\pdfoutput=1
\documentclass{article}
\PassOptionsToPackage{compress,numbers}{natbib}
\usepackage[preprint]{neurips_2025}
\usepackage{times}
\usepackage[utf8]{inputenc} 
\usepackage[T1]{fontenc}
\usepackage{url}            
\usepackage{booktabs}       
\usepackage{amsfonts}       
\usepackage{nicefrac}       
\usepackage{microtype}      
\usepackage{graphicx}
\usepackage{subfigure}
\usepackage{booktabs} 
\usepackage{bbm}
\usepackage{amsmath}
\usepackage{amssymb}
\usepackage{mathtools}
\usepackage{amsthm}
\usepackage{wrapfig}
\usepackage{paralist}
\usepackage{adjustbox}
\usepackage{calc}
\usepackage{algorithm}
\usepackage{algorithmic}
\usepackage{enumitem}
\usepackage{anyfontsize}
\usepackage[dvipsnames]{xcolor}

\def\I{{\bf I}}

\def\x{{\bf x}}
\def\y{{\bf y}}
\def\z{{\bf z}}

\def\0{{\bf 0}}
\def\1{{\bf 1}}

\newtheorem{lemma}{Lemma}

\newtheorem{proposition}{Proposition}

\numberwithin{theorem}{section}
\numberwithin{lemma}{section}
\numberwithin{remark}{section}
\numberwithin{cor}{section}
\numberwithin{proposition}{section}
\numberwithin{definition}{section}

\newcommand{\tabref}[1]{Table~\ref{#1}}
\newcommand{\secref}[1]{Sec.~\ref{#1}}
\newcommand{\appref}[1]{Appendix~\ref{#1}}
\newcommand{\figref}[1]{Fig.~\ref{#1}}
\newcommand{\lemref}[1]{Lemma~\ref{#1}}
\newcommand{\prpref}[1]{Proposition~\ref{#1}}

\newcommand{\eqnref}[1]{Eqn.~(\ref{#1})}
\newcommand{\algref}[1]{Alg.~\ref{#1}}

\renewcommand{\tilde}{\widetilde}

\renewcommand{\frac}{\dfrac}
\def\cite#1{\citep{#1}}

\usepackage[colorlinks=true,citecolor=brown,urlcolor=gray]{hyperref}

\def\blue#1{\textcolor{blue}{#1}}
\def\green#1{\textcolor{green}{#1}}

\author{%
  Ziyan Wang \\
  Georgia Institute of Technology \\
  \texttt{wzy@gatech.edu}
  \And
  Sizhe Wei \\
  Georgia Institute of Technology \\
  \texttt{swei@gatech.edu}\\
  \And
  Xiaoming Huo \\
  Georgia Institute of Technology \\
  \texttt{huo@gatech.edu}\\
  \And
  Hao Wang \\
  Rutgers University \\
  \texttt{hw488@cs.rutgers.edu}\\
}

\title{PoGDiff: Product-of-Gaussians Diffusion Models for Imbalanced Text-to-Image Generation}

\begin{document}

\maketitle

\begin{abstract}
Diffusion models have made significant advancements in recent years. 
However, their performance often deteriorates when trained or fine-tuned on imbalanced datasets. 
This degradation is largely due to the disproportionate representation of majority and minority data in image-text pairs. 
In this paper, we propose a general fine-tuning approach, dubbed PoGDiff, to address this challenge. 
Rather than directly minimizing the KL divergence between the predicted and ground-truth distributions, PoGDiff replaces the ground-truth distribution with a Product of Gaussians (PoG), which is constructed by combining the original ground-truth targets with the predicted distribution conditioned on a neighboring text embedding. 
Experiments on real-world datasets demonstrate that our method effectively addresses the imbalance problem in diffusion models, improving both generation accuracy and quality.
\end{abstract}

\section{Introduction} \label{sec:intro}
The development of diffusion models~\citep{ho2020denoising, song2020score} and their subsequent extensions~\citep{song2020denoising, nichol2021improved, huang2023conditional} has significantly advanced the learning of complex probability distributions across various data types, including images~\citep{ho2022cascaded, rombach2022high, saharia2022photorealistic, ho2022classifier}, audio~\citep{kong2020diffwave}, and 3D biomedical imaging data~\citep{luo2021diffusion, poole2022dreamfusion, shi2023mvdream, pinaya2022brain}. 
For these generative models, the amount of training data plays a critical role in determining both the accuracy of probability estimation and the model’s ability to generalize, which enables effective extrapolation within the probability space. 

Data diversity and abundance are key to improving the generative capabilities of large-scale models, enabling them to capture intricate details within a vast probability space. 
However, many data-driven modeling tasks often rely on small, imbalanced real-world datasets, leading to poor generation quality, particularly for minority groups. 
For example, when training and fine-tuning a diffusion model with an imbalanced dataset of individuals, existing models often struggle to generate accurate images for those who appear less frequently (i.e., minorities) in the training data (\figref{fig:pogdiff_main}).
This challenge is further compounded when accuracy is prioritized over simply high resolution. 
For example, generated images of individuals need to match the identity of at least one individual in the training set (\figref{fig:pogdiff_main}). 
Addressing this gap is crucial for deploying diffusion models in real-world applications where correctness is paramount, such as personalized content generation or medical imaging. 

\begin{figure}
\centering
\includegraphics[width=0.99\linewidth]{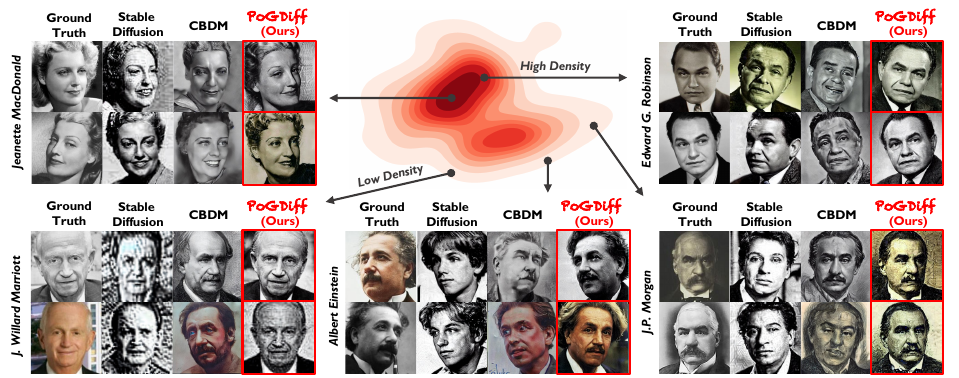}
\vspace{-2mm}   
\caption{\textbf{PoGDiff for imbalanced text-to-image generation.} Existing methods, e.g., Stable Diffusion~\citep{rombach2022high} and CBDM~\citep{qin2023class}, fall short for minority data (\textbf{Low Density}). In contrast, Our PoGDiff successfully generates high-quality images even for minority data, outperforming all baselines.}  
\label{fig:pogdiff_main}
\end{figure}
This limitation is true even for finetuning large diffusion models pretrained on large-scale datasets like LAION-5B~\citep{schuhmann2022laion}, e.g., Stable Diffusion~\citep{rombach2022high}. 
Imagine an imbalanced dataset consisting of employees in a small company, senior employees might have more photos available, while new employees only have a very limited number of them. 
Since none of the employees appear in the LAION-5B dataset, generating photos of them requires finetuning the Stable Diffusion model. 
Unfortunately, finetuning the model on such an imbalanced dataset might enable the model to generate accurate images for the majority group (i.e., senior employees), but it will perform poorly for the minority group (i.e., new employees). 

To address this challenge, we propose a general fine-tuning approach, dubbed PoGDiff. 
Rather than directly minimizing the KL divergence between the predicted and ground-truth distributions, PoGDiff replaces the ground-truth distribution with a Product of Gaussians (PoG), which is constructed by combining the original ground-truth targets with the predicted distribution conditioned on a neighboring text embedding. 
Our contributions are as follows:
\begin{itemize}[nosep]
\item We identify the problem of imbalanced text-to-image generation (IT2I) and introduce the first general diffusion model, dubbed Product-of-Gaussians Diffusion Models (PoGDiff), to address this problem. 

\item Our theoretical analysis shows that training of PoGDiff is equivalent to training a normal diffusion model while encouraging the model to generate the same image given similar text prompts (conditions). 

\item {We propose a new metric, ``Generative Recall'' (gRecall), which evaluates the generative diversity of a model when generation accuracy is strictly enforced.}

\item Our empirical results on real-world datasets demonstrate the effectiveness of our method, outperforming all state-of-the-art baselines.
\end{itemize}
\section{Related Work}
\textbf{Long-Tailed Recognition.} 
Addressing the challenges posed by long-tailed data distributions has been a critical area of research in machine learning, for both classification and regression problems. 
Traditional methods, such as re-sampling and re-weighting techniques, have been used to mitigate class imbalances by either over-sampling minority classes or assigning higher weights to them during training~\citep{chawla2002smote, he2009learning, torgo2013smote, branco2017smogn, branco2018rebagg}. 
Such algorithms fail to measure the distance in continuous label space and fail to handle high-dimensional data (e.g., images, and text). 
Deep imbalanced regression methods~\cite{yang2021delving,ren2022balanced,gong2022ranksim,keramati2023conr,wang2024variational} address this challenge by reweighting the data using the effective label density during representation learning. 
However, all methods above are designed for \emph{recognition} tasks such as classification and regression and are therefore not applicable to our \emph{generation} task. 

\textbf{Diffusion Models Related to Long-Tailed Data.} 
There are also works related to both diffusion models and long-tailed data. They aim at improving generation robustness using feature engineering~\citep{ruiz2023dreambooth,kumari2023multi,li2024photomaker}, feature augmentation~\citep{zhang2024long}, noisy label~\cite{na2024label}, improving fairness in image generation~\cite{shen2023finetuning,kim2024training,choi2020fair}, and improving classification accuracy using diffusion models~\cite{zhang2024long}. However, these works have different goals and, therefore, are not applicable to our setting.

Most relevant to our work is the Class Balancing Diffusion Model (CBDM)~\citep{qin2023class}, which uses a distribution adjustment regularizer that enhances tail-class generation based on the model’s predictions for the head class. 
It improves the quality of long-tailed generation by assuming one-hot conditional labels (i.e., classification-based settings). 
However, this assumption does not generalize to the modern setting where image generation is usually conditioned on free-form text prompts. 
As a result, when adapted to the free-form setting, they often fail to model the similarity among different text prompts, leading to suboptimal generation performance in minority data (as verified by empirical results in~\secref{sec:experiments}). 

\section{Methods}
\subsection{Preliminaries}
\textbf{Diffusion models (DMs)}~\citep{ho2020denoising} are probabilistic models that generate an output image $\x_0$ from a random noise vector $\x_T$ conditioned on text input $\mathbf{c}$. 
DMs operate through two main processes: the forward diffusion process and the reverse denoising process. 
During the diffusion process, Gaussian noise is progressively added to a data sample $\x_0$ over $T$ steps. The forward process is defined as a Markov chain, where:
\begin{align*}
    q \left( \x_t | \x_{t-1} \right) = \mathcal{N} \left( \x_t; \sqrt{1 - \beta_t} \x_{t-1}, \beta_t \I \right).
\end{align*}
Here, $\beta_t$ is the predefined diffusion rate at step $t$. By denoting $\alpha_t = 1 - \beta_t$ and $\Bar{\alpha}_t = \prod\nolimits_{i=1}^{t} \alpha_i$, we can describe the entire diffusion process as:
\begin{align*}
    q \left( \x_{1:T} | \x_{0} \right) = \prod\nolimits_{t=1}^{T} q \left( \x_t | \x_{t-1} \right), \quad 
    q \left( \x_t | \x_{0} \right) = \mathcal{N} \left( \x_t; \sqrt{\Bar{\alpha}_t} \x_{0}, (1 - \Bar{\alpha}_t) \I \right)
\end{align*}
The denoising process removes noise from the sample $\x_T$, eventually recovering $\x_0$.
A denoising model $\epsilon_{\theta} (\x_t, t, \y)$ is trained to estimate the noise $\epsilon$ from $\x_t$ and a text-guided embedding $\y = \phi (\mathbf{c})$, where $\phi (\cdot)$ is a pretrained text encoder. Formally:
\begin{align*}
    p_{\theta} \left( \x_{t-1} | \x_{t}, t, \y \right) &= \mathcal{N} \left( \x_{t-1}; \epsilon_{\theta} (\x_t, t, \y), \sigma_{t}^{2} \I \right).
\end{align*}
The denoising process is trained by maximizing the likelihood of the data under the model or, equivalently, by minimizing the variational lower bound on the negative log-likelihood of the data. 
~\citet{ho2020denoising} shows that this is equivalent to minimizing the KL divergence between the predicted distribution $p_{\theta} (\x_{t-1} | \x_{t}, \y)$ and the ground-truth distribution $q (\x_{t-1} | \x_{t}, \x_{0}, \y)$ at each time step $t$ during the backward process. The training objective then becomes:
\begin{align*}
    \min\nolimits D_{KL} \left( q \left( \x_{t-1} | \x_{t}, \x_{0}, \y \right) \big\| p_{\theta} \left( \x_{t-1} | \x_{t}, \y \right) \right).
\end{align*}
This can be simplified to {(i.e., Eqn. (1) in~\citet{rombach2022high})}:
\begin{align*}
    L_{DM} &= \mathbb{E}_{\x, \epsilon \sim \mathcal{N}(0,\I),t} \left[\| \epsilon - \epsilon_{\theta} (\x_t, t, \y) \|_{2}^{2}\right]. 
\end{align*}
\textbf{Latent diffusion models (LDMs)}~\citep{rombach2022high} are diffusion models that perform the entire diffusion and denoising process in a lower-dimensional latent space. LDMs first learn an encoder $\mathcal{E}$ and a decoder $\mathcal{D}$, which are then frozen during subsequent training of the diffusion models. The corresponding objective is then simplified to {(i.e., Eqn. (2) in~\citet{rombach2022high})}: 
\begin{align*}
    L_{LDM} = \mathbb{E}_{\mathcal{E}(\x), \epsilon \sim \mathcal{N}(0,\I),t} \left[\| \epsilon - \epsilon_{\theta} (\z_t, t, \y) \|_{2}^{2}\right]
\end{align*}
In this paper, we use Stable Diffusion (SD)~\citep{rombach2022high} as our backbone model. Since our method works for both the vanilla DMs and LDMs, for clarity, we use the notation $\x$ instead of $\z$, as the encoder $\mathcal{E}$ and decoder $\mathcal{D}$ are fixed during fine-tuning.

\subsection{Product-of-Gaussians Diffusion Models (PoGDiff)}

\subsubsection{Main Idea}
\textbf{Method Overview.}  
Given an image dataset $\mathcal{D} = \{ \x^{(i)}, \mathbf{c}^{(i)} \}_{i=1}^{N}$, where $\mathbf{c}^{(i)}$ is the text description for image $\x^{(i)}$, we use a fixed CLIP encoder to produce $\mathbf{c}^{(i)}$'s corresponding text embedding $\y=\phi(\mathbf{c})$. 

Typical diffusion models minimize the KL divergence between the predicted distribution $p_{\theta} (\x_{t-1} | \x_{t}, \y) = \mathcal{N} ({\x_{\theta}} (\x_t, t, \y), \lambda_{\y}^{-1} \I)$ and the ground-truth distribution $q (\x_{t-1} | \x_{t}, \x_{0}, \y) = \mathcal{N} ({\x_{t-1}}, \lambda_{t}^{-1} \I)$ at each time step $t$ during the backward denoising process. 
Here, $\lambda_{\y}$ and $\lambda_{t}$ represent the precision. 
In contrast, our PoGDiff replaces the ground-truth target with a Product of Gaussians (PoG) and instead minimizes the following KL divergence (for each $t$)
\begin{align}
     \mathcal{L}_{t-1}^{\textrm{PoGDiff}} = D_{KL} \big( q \left(\x_{t-1} | \x_{t}, \x_{0}, \y \right) \circ p_{\theta} \left(\x_{t-1} | \x_{t}, \y' \right) \big\| p_{\theta} \left(\x_{t-1} | \x_{t}, \y \right) \big), \label{eq:klpog}
\end{align}
\begin{wrapfigure}{t}{0.5\textwidth}
\centering
\vskip -0.1cm
\includegraphics[width=0.45\textwidth]{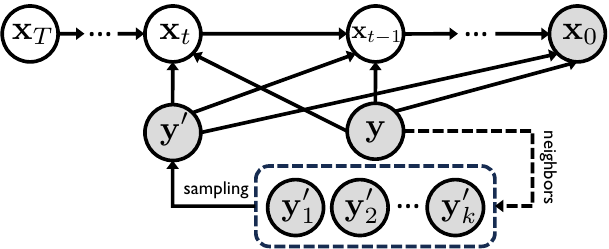}
\vskip -0.3cm
    \caption{Overview of our PoGDiff. During finetuning, PoGDiff collects $k$ neighbors of the current text embedding $\y$ and samples one $\y'$ from them based on~\eqnref{eq:KNN-weights}. Both $\y$ and $\y'$ will then be employed to denoise the current image $\x_t$ to $\x_{t-1}$.}
\vskip -0.6cm
\label{fig:pogdiff_gc}
\end{wrapfigure}
where $\circ$ represents the product of two Gaussian distributions, $\y'$ is a selected neighboring embedding from other samples in the training dataset (more details below), and $p_{\theta} (\x_{t-1} | \x_{t}, \y')$ denotes the predicted distribution when using $\y'$ as the input text embedding. 

As shown in~\figref{fig:pogdiff_gc}, intuitively, PoGDiff's denoising model $\epsilon_{\theta} (\x_t, t, \y)$ (or $p_{\theta} (\x_{t-1} | \x_{t}, \y)$) is optimized towards two target distributions, equivalently increasing the weights for minority instances (more details below). 
This enhances the text-to-image mapping by leveraging the statistical strength of neighboring data points, thereby improving and quality of the generated images, especially for minority images.

\textbf{Intuition behind the Product of Gaussians (PoG).} 
During fine-tuning, typical diffusion models ``lock'' the text conditional embedding $\y = \psi (c)$ to the corresponding image $\x$. 
Consequently, if the dataset follows a long-tailed distribution, the fine-tuned or post-trained diffusion model becomes heavily biased toward the majority data, performing poorly on minority data. 
\figref{fig:pogdiff_core} demonstrates our intuition. When training using a text-image pair $(\y,\x)$, our PoGDiff ``borrows'' information from neighboring text conditional embedding $\y'$, thereby effectively increasing the data density in the minority region and leading to smoother (less imbalanced) effective density, as shown in \figref{fig:pogdiff_core} (right). 
However, since the text embedding is fixed during fine-tuning (i.e., $\phi$ is frozen), directly smoothing the text embedding space is not feasible. Instead, we rely on the properties of PoG. 

By definition, given two Gaussian distributions, $\mathcal{N} (\mu_1, \lambda_{1}^{-1} \I)$ and $\mathcal{N} (\mu_2, \lambda_{2}^{-1} \I)$, their product is still a Gaussian distribution:
\begin{align}
    \mathcal{N} (\mu_1, \lambda_{1}^{-1}) \circ \mathcal{N} (\mu_2, \lambda_{2}^{-1}) = \mathcal{N} \left( \frac{\lambda_1 \mu_1 + \lambda_2 \mu_2}{\lambda_1 +\lambda_2}, (\lambda_1 + \lambda_2)^{-1} \right) \triangleq \mathcal{N} \left( \mu_{\textrm{PoG}}, \lambda_{\textrm{PoG}}^{-1} \right), \label{eq:pog_def}
\end{align}
which can be treated as a ``composition'' of two individual Gaussians, incorporating information from both. 
This intuition is key to developing our PoGDiff objective function.

\subsubsection{Theoretical Analysis and Algorithmic Design}
Based on \eqnref{eq:klpog}, we then derive a concrete objective function following~\prpref{prp:objective} below. 

\begin{proposition}\label{prp:objective}
Assume $\lambda_{\y} = \lambda_{\textrm{PoG}} \triangleq \lambda_t + \lambda_{\y'}$, we have our loss function
\begin{align}
    \mathcal{L}_{t-1}^{\textrm{PoGDiff}} = \mathbb{E}_{q} \left[ 
    \frac{\lambda_{\y}}{2} \left\| \mu_{\theta} (\x_{t}, \y) - \mu_{\textrm{PoG}} \right\|^2 \right] + C.
\end{align}
Here, $C$ is a constant, and $\mu_{\textrm{PoG}}$ denotes the mean of the PoG, with the expression defined in~\eqnref{eq:pog_def}. 
Then, through derivations based on Gaussian properties, we obtain
\begin{align}
    \mathcal{L}_{t-1}^{\textrm{PoGDiff}} 
    \leq \mathbb{E}_{q} \Big[ 
    \mathcal{A} (\lambda_t) \left\| \epsilon_{\theta} (\x_{t}, \y) - \epsilon \right\|^2 + \mathcal{A} (\lambda_{\y'}) \left\| \epsilon_{\theta} (\x_{t}, \y) - \epsilon_{\theta} (\x_{t}, \y') \right\|^2 \Big] + C \label{eq:obj_ori}
\end{align}
where the function $\mathcal{A} (\lambda) \triangleq \frac{\lambda (1 - \alpha_t)^2}{2 \alpha_t (1 - \Bar{\alpha}_t)}$.
\end{proposition}

The proof is available in ~\appref{app:proof}. \eqnref{eq:obj_ori} in \prpref{prp:objective} provides a upper bound for the KL divergence (\eqnref{eq:klpog}) we aim to minimize. 

In diffusion model literature~\citep{ho2020denoising, rombach2022high}, one typically sets $\mathcal{A} (\lambda_t) = 1$ to eliminate the dependency on the time step $t$, and thus~\eqnref{eq:obj_ori} can be written as\footnote{For clarification, our $\mathcal{A} (\lambda_t)$ is equivalent to $\lambda_t$ in~\citep{ho2020denoising}, with the difference that in our paper, $\lambda$ refers to the precision of the Gaussian distribution.}:
\begin{align}
    \mathcal{L}_{\textrm{simple}}^{\textrm{PoGDiff}} = \mathbb{E}_{\x \sim q(\x_0), \epsilon\sim\mathcal{N}(\0,\I), t \sim \mathcal{U}(1,T)} \left[ 
    \left\| \epsilon_{\theta} (\x_{t}, \y) - \epsilon \right\|^2 + \frac{\lambda_{\y'}}{\lambda_t} \left\| \epsilon_{\theta} (\x_{t}, \y) - \epsilon_{\theta} (\x_{t}, \y') \right\|^2 \right]. 
\end{align}
For convenience, we rewrite $\tfrac{\lambda_{\y'}}{\lambda_t}=\tfrac{\sigma_{t}^{2}}{\sigma_{\y'}^{2}}$. 
Note that this weight still depends on the time step $t$. 
{Therefore, to be consistent with the literature~\citep{ho2020denoising, rombach2022high}, we hypothetically define $\sigma_{\y'}^{2} = \tfrac{\sigma_{t}^{2}}{\psi \left[ \left( \x, \y \right), \left( \x', \y' \right) \right]}$ to cancel out the term $\sigma_{t}^{2}$, thereby effectively removing the time step dependency;} here $\psi \left[ \left( \x, \y \right), \left( \x', \y' \right) \right]$ denotes the similarity between the two image-text pairs. 
By shortening the notation $\psi \left[ \left( \x, \y \right), \left( \x', \y' \right) \right]$ to $\psi$, we can further rewrite the objective function for PoGDiff as:
\begin{align}
    \mathcal{L}_{\textrm{simple}}^{\textrm{PoGDiff}} = \mathbb{E}_{\x \sim q(\x_0), \epsilon\sim\mathcal{N}(\0,\I), t \sim \mathcal{U}(1,T)} \Big[ 
    \left\| \epsilon_{\theta} (\x_{t}, \y) - \epsilon \right\|^2 + \psi \left\| \epsilon_{\theta} (\x_{t}, \y) - \epsilon_{\theta} (\x_{t}, \y') \right\|^2 \Big]. \label{eq:pog_simple}
\end{align}
\subsubsection{Computing the Similarity \texorpdfstring{$\psi$}{}} 
Next, we discuss the choice of $\psi$ in~\eqnref{eq:pog_simple}. 
Given an image-text dataset $\mathcal{D}$, the similarities between each image-text pair need to be considered in two parts:
\begin{wrapfigure}{t}{0.5\textwidth}
\centering
\vskip -0.1cm
\includegraphics[width=0.45\textwidth]{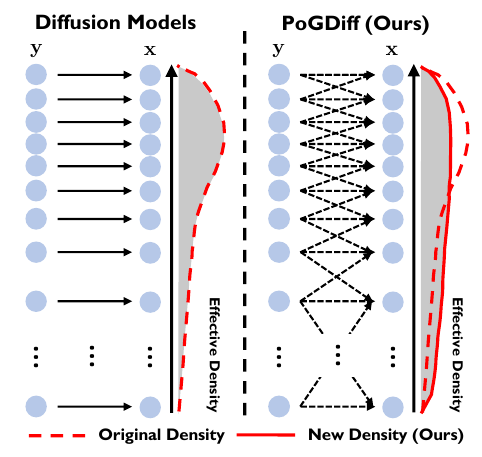}
\vskip -0.3cm
\caption{Comparing denoising networks of typical diffusion models~\citep{ho2020denoising,rombach2022high} and our PoGDiff. \textbf{Left:} In {conditional text-to-image diffusion models}, a data point (i.e., $\x$) is mainly affected by its text embedding ({besides random latent codes}). \textbf{Right:} In PoGDiff, neighbors participate to modulate the final effective density. {Here, $\y$ denotes the text prompts, which are the embeddings of the text descriptions of the images; $\x$ denotes the associated images. The tightly packed circles at the top indicate higher density, while the sparse circles indicate lower density.}}
\vskip -0.7cm
\label{fig:pogdiff_core}
\end{wrapfigure}
\begin{align}
    \psi \triangleq \psi_{\textrm{img-sim}} \left( \x, \x' \right) \cdot \psi_{\textrm{inv-txt-den}} \left( \y \right), \label{eq:psi}
\end{align}
where $\psi_{\textrm{img-sim}}( \x, \x' )$ is the similarity between images $\x$ and $\x'$, and $\psi_{\textrm{inv-txt-den}} ( \y )$ is the probability density of the text embedding $\y$ (more details below). 

\textbf{Image Similarity $\psi_\textrm{img-sim}$.} 
For all $\x \sim \mathcal{D}$, we apply a pre-trained image encoder to obtain the latent representations $\z$. We then calculate the cosine similarities between each $\z$ and select the $k$ nearest neighbors with the highest similarity values for all samples in the dataset $\mathcal{D}$, denoted as $\left[ s_{j} \right]_{j=1}^{k}$, where $s_{j}$ represents the cosine similarity scores between $\x$ and other images in $\mathcal{D}$, sorted in descending order. These values are then normalized to produce the weights for each neighbor:  
\begin{align}
    w_{j} = \frac{s_{j}}{\sum_{j} s_{j}}. \label{eq:KNN-weights}
\end{align}
For each data pair $(\x, \y)$, we then randomly sample a neighboring pair $(\x', \y')$ through from a categorical distribution $Cat([w_j]_{j=1}^k)$ (``Cat'' is short for ``Categorical''),
i.e., with $w_{j}$ serving as the probability weight, and compute their image similarity as:
\begin{align}
    \psi_{\textrm{img-sim}} \left( \x, \x' \right) \triangleq \max \left( 0, s^{a_1 + a_2 \cdot \mathbbm{1} \left[ \mathcal{I} (\x) \neq \mathcal{I} (\x') \right]} \right), \label{eq:img-sim}
\end{align}
where {$s\in\{s_j\}_{j=1}^k$ denotes the cosine similarity sampled according to the weights $\{w_j\}_{j=1}^k$ in~\eqnref{eq:KNN-weights}},  $\mathbbm{1} \left[ \cdot \right]$ denotes the indicator function, and $\mathcal{I} (\cdot)$ retrieves the class/identity of the current input image; 
$\mathbbm{1} \left[ \mathcal{I} (\x) \neq \mathcal{I} (\x') \right]=0$ if $\x$ and $\x'$ are two photos of the same person (e.g., Albert Einstein), and $\mathbbm{1} \left[ \mathcal{I} (\x) \neq \mathcal{I} (\x') \right]=1$ if $\x$ and $\x'$ are photos of two different persons (e.g., $\x$ is Einstein and $\x'$ Reagan). {More details on $\mathcal{I}(\cdot)$ can be found in~\appref{app:psi_var}.}
$a_1, a_2$ are hyperparameters that control the scale of the similarities. 
For example, if the cosine similarity ($s$) between $x$ and $x'$ is 0.4, and $a_1=a_2=1$: if $x$ and $x'$ are of the same person, the image similarity will be $0.4^{1}$, whereas if $x$ and $x'$ are not of the same person, the image similarity will be $0.4^{2}$, which is smaller. 
The intuition is to compute the image similarity according to both the image content similarity, i.e., $s$, and identity similarity, i.e., $\mathcal{I} (\x)$ and $\mathcal{I} (\x')$. 
\begin{algorithm}[t]
\caption{Training Algorithm for PoGDiff}\label{alg:PoGDiff}
\begin{algorithmic}[1]
\STATE {\bfseries Inputs:} A dataset $\mathcal{D} = \{ \x^{(i)}, \mathbf{c}^{(i)} \}_{i=1}^{N}$.\\
\STATE {\bfseries repeat}\\
\STATE {\quad $(\x_0, \mathbf{c}) \sim \mathcal{D}$}\\
\STATE {\quad $\y = \phi (\mathbf{c})$}\\
\STATE {\quad $t \sim \textrm{Uniform} (1, \cdots, T)$}\\
\STATE {\quad $\epsilon \sim \mathcal{N} (\0, \I)$}\\
\STATE {\quad Generate $\y'$ and $\psi$ from~\eqnref{eq:psi_final}}\\
\STATE {\quad Calculate $\x_t = \sqrt{\Bar{\alpha}_t} \x_0 + \sqrt{1 - \Bar{\alpha}_t} \epsilon$}\\
\STATE {\quad Take gradient descent step on}\\
\STATE {\small \quad \quad $\nabla_{\theta} \left[ \| \epsilon - \epsilon_{\theta} (\x_t, \y) \|_{2}^{2} + \psi \| \epsilon_{\theta} (\x_t, \y') - \epsilon_{\theta} (\x_t, \y) \|_{2}^{2} \right]$}\\
\STATE {\bfseries until} converged
\end{algorithmic}
\end{algorithm}

\textbf{Inverse Text Densities $\psi_\textrm{inv-txt-den}$.} 
Inspired by LDS in DIR~\citep{yang2021delving} and the theoretical analysis in VIR~\citep{wang2024variational}, re-weighting the label distribution of an imbalanced dataset can increase the optimization scale for minority classes and reduce the emphasis on majority classes, resulting in better performance under imbalanced conditions. 
However, both DIR and VIR partition the label space into bins, treating it as a classification problem. 
This is \emph{not applicable} to our setting because in text-to-image generation, the ``label'' is actually text embeddings. 
Instead, we train a variational autoencoder (VAE) on this dataset and then approximate its likelihood $p(\y)$ through its evidence lower bound, or ELBO:
\begin{align}
    p(\y) = e^{\log p (\y)} \approx e^{\textrm{ELBO}_{\textrm{VAE}} (\y)} .
\end{align}
The evidence for minority data will be lower than for majority classes. 
This then motivates our inverse text densities defined as follows:
\begin{align}
    \psi_{\textrm{inv-txt-den}} \left( \y \right) \triangleq \frac{1}{a_3} e^{-\textrm{ELBO}_{\textrm{VAE}} (\y)}, \label{eq:inv-txt-den}
\end{align}
where $a_3$ is a hyperparameter that controls the scale of the inverse text densities. 
By combining~\eqnref{eq:img-sim} and~\eqnref{eq:inv-txt-den} to~\eqnref{eq:psi}, we can then compute $\psi$ as follows:
\begin{align}
    \psi = \max \Big( 0, \frac{s^{a_1 + a_2 \cdot \mathbbm{1} \left[ \mathcal{I} (\x) \neq \mathcal{I} (\x') \right]}}{a_3}  \Big) \cdot e^{-\textrm{ELBO}_{\textrm{VAE}} (\y)}. \label{eq:psi_final}
\end{align}
\subsubsection{Final Objective Function}
By collecting all the components discussed above, we arrive at our final training objective:
\begin{align}
    \mathcal{L}_{\textrm{final}}^{\textrm{PoGDiff}} &= \mathbb{E}_{\x \sim q(\x_0), \epsilon\sim\mathcal{N}(\0,\I), t \sim \mathcal{U}(1,T)} \Big[ 
    \left\| \epsilon_{\theta} (\x_{t}, \y) - \epsilon \right\|^2 + \psi \left\| \epsilon_{\theta} (\x_{t}, \y) - \epsilon_{\theta} (\x_{t}, \y') \right\|^2 \Big], 
\end{align}
where $\psi$ is defined in~\eqnref{eq:psi_final}. \algref{alg:PoGDiff} summarizes our algorithm.

\section{Experiments}\label{sec:experiments}
\subsection{Experimental Setup}\label{sec:setup}
{\textbf{Datasets.} To demonstrate the effectiveness of PoGDiff in terms of both accuracy and quality, we evaluate our method on four widely used imbalanced datasets, i.e., AgeDB~\citep{moschoglou2017agedb}, DigiFace~\citep{bae2023digiface}, VGGFace2~\citep{cao2018vggface2} and CIFAR-100-LT~\citep{cao2019learning}.} 


{\emph{AgeDB-IT2I-L \& AgeDB-IT2I-S:} AgeDB-IT2I-L(arge) is constructed from the AgeDB dataset~\citep{moschoglou2017agedb}. 
It consists of $976$ images across $223$ identities, with each majority class containing $30$ images and each minority class containing $2$ images. 
We also construct \emph{AgeDB-IT2I-S}(mall), which contains $32$ images across $2$ identities, where each majority class consists of $30$ images and each minority class consists of $2$ images. 
Additionally, we construct \emph{AgeDB-IT2I-M}(edium), and more details can be found in~\appref{app:datasets}. }

{\emph{DigiFace-IT2I:} DigiFace-IT2I is derived from the DigiFace dataset~\citep{bae2023digiface}. It contains $985$ images across $179$ identities, where each majority class consists of $30$ images and each minority class consists of $2$ images. 
We use a process similar to AgeDB-IT2I to collect text-image pairs, forming this DigiFace-IT2I dataset. }

{\emph{VGGFace-IT2I:} VGGFace-IT2I is a subset of VGGFace2~\citep{cao2018vggface2}. 
It contains $1933$ images across $193$ identities, where each majority class consists of $49$ images and each minority class consists of $2$ images. }

{\emph{CIFAR-100-IT2I:} CIFAR-100-IT2I is constructed from CIFAR-100-LT dataset~\citep{cao2019learning}. 
We set the imbalance ratio to 250, where the largest class contains $500$ images and the smallest class contains only $2$ images. 
The dataset consists of $9502$ images in total. 
During fine-tuning, to mitigate the dominance of majority classes, we downsample these classes in each epoch. Specifically, the most frequent class is reduced to $49$ images per epoch out of its original $500$ images.}

\begin{table}[t]
\begin{adjustbox}{valign=t}
\centering
\begin{minipage}{0.49\textwidth}
\centering
\setlength{\tabcolsep}{2pt}
\caption{\textbf{Performance based on FID score.}}
\vspace{-0.5cm}
\label{table:total-fid}
\small
\begin{center}
\resizebox{\textwidth}{!}{
\begin{tabular}{l|cc|cc|cc|cc|cc}
\toprule[1.5pt]
Datasets (-IT2I) & \multicolumn{4}{c|}{AgeDB} & \multicolumn{2}{c|}{DigiFace} & \multicolumn{2}{c|}{{VGGFace}} & \multicolumn{2}{c}{{CIFAR-100}} \\ \midrule
Size & \multicolumn{2}{c|}{Small} & \multicolumn{2}{c|}{Large} & \multicolumn{2}{c|}{Large} & \multicolumn{2}{c|}{Large} & \multicolumn{2}{c}{Large} \\ \midrule
Metric & \multicolumn{10}{c}{FID~$\downarrow$} \\ \midrule
Shot & All & Few & All & Few & All & Few & All & Few & All & Few \\ \midrule
\textsc{Vanilla} & 14.88 & 13.72 & 7.67 & 11.67 & 7.18 & 12.23 & 7.59 & 12.08 & 7.19 & 11.46 \\[1.5pt]
\textsc{CBDM} & 14.72 & 14.13 & 7.18 & 11.12 & 6.96 & 12.72 & 7.23 & 11.91 & 7.26 & 11.94 \\[1.5pt]
{\textsc{T2H}} & 14.85 & 13.66 & 7.61 & 11.64 & 7.14 & 12.22 & 7.34 & 12.02 & 7.10 & 11.39 \\[1.5pt]
\textsc{PoGDiff (Ours)} & \textbf{14.15} & \textbf{12.88} & \textbf{6.03} & \textbf{10.16} & \textbf{6.84} & \textbf{11.21} & \textbf{6.29} & \textbf{10.97} & \textbf{6.24} & \textbf{9.41} \\[1.5pt]
\bottomrule[1.5pt]
\end{tabular}}
\end{center}
\vspace{-0.5cm}
\end{minipage}
\end{adjustbox}
\hfill
\begin{adjustbox}{valign=t}
\centering
\begin{minipage}{0.49\textwidth}
\centering
\setlength{\tabcolsep}{2pt}
\caption{\textbf{Performance based on DINO score.}}
\vspace{-0.5cm}
\label{table:total-dino} 
\small
\begin{center}
\resizebox{\textwidth}{!}{
\begin{tabular}{l|cc|cc|cc|cc|cc}
\toprule[1.5pt]
Datasets (-IT2I) & \multicolumn{4}{c|}{AgeDB} & \multicolumn{2}{c|}{DigiFace} & \multicolumn{2}{c|}{{VGGFace}} & \multicolumn{2}{c}{{CIFAR-100}} \\ \midrule
Size & \multicolumn{2}{c|}{Small} & \multicolumn{2}{c|}{Large} & \multicolumn{2}{c|}{Large} & \multicolumn{2}{c|}{Large} & \multicolumn{2}{c}{Large} \\ \midrule
Metric & \multicolumn{10}{c}{DINO~$\uparrow$} \\ \midrule
Shot & All & Few & All & Few & All & Few & All & Few & All & Few \\ \midrule
\textsc{Vanilla} & 0.42 & 0.37 & 0.34 & 0.25 & 0.42 & 0.36 & 0.41 & 0.29 & 0.48 & 0.32 \\[1.5pt]
\textsc{CBDM} & 0.54 & 0.09 & 0.41 & 0.26 & 0.34 & 0.16 & 0.46 & 0.22 & 0.52 & 0.28 \\[1.5pt]
{\textsc{T2H}} & 0.43 & 0.39 & 0.37 & 0.26 & 0.44 & 0.36 & 0.42 & 0.28 & 0.45 & 0.30 \\[1.5pt]
\textsc{PoGDiff (Ours)} & \textbf{0.77} & \textbf{0.73} & \textbf{0.66} & \textbf{0.52} & \textbf{0.64} & \textbf{0.49} & \textbf{0.69} & \textbf{0.55} & \textbf{0.73} & \textbf{0.61} \\[1.5pt]
\bottomrule[1.5pt]
\end{tabular}}
\end{center}
\vspace{-0.5cm}
\end{minipage}
\end{adjustbox}
\end{table} 
\begin{table}[t]
\begin{adjustbox}{valign=t}
\centering
\begin{minipage}{0.49\textwidth}
\centering
\setlength{\tabcolsep}{2pt}
\caption{\textbf{Performance based on human evaluation.} The evaluation is a binary decision: the image is either judged as representing the same individual (score 1.0) or not (score 0.0).}
\vspace{-0.6cm}
\label{table:total-human}   
\small
\begin{center}
\resizebox{\textwidth}{!}{
\begin{tabular}{l|cc|cc|cc|cc}
\toprule[1.5pt]
Datasets (-IT2I) & \multicolumn{4}{c|}{AgeDB} & \multicolumn{2}{c|}{{VGGFace}} & \multicolumn{2}{c}{{CIFAR-100}} \\ \midrule
Size & \multicolumn{2}{c|}{Small} & \multicolumn{2}{c|}{Large} & \multicolumn{2}{c|}{Large} & \multicolumn{2}{c}{Large} \\ \midrule
Metric & \multicolumn{8}{c}{Human Score~$\uparrow$} \\ \midrule
Shot & All & Few & All & Few & All & Few & All & Few \\ \midrule
\textsc{Vanilla} & 0.50 & 0.00 & 0.60 & 0.20 & 0.62 & 0.16 & 0.72 & 0.30 \\[1.5pt]
\textsc{CBDM} & 0.50 & 0.00 & 0.56 & 0.12 & 0.54 & 0.10 & 0.63 & 0.24 \\[1.5pt]
{\textsc{T2H}} & 0.50 & 0.00 & 0.60 & 0.20 & 0.62 & 0.16 & 0.72 & 0.30 \\[1.5pt]
\textsc{PoGDiff (Ours)} & \textbf{1.00} & \textbf{1.00} & \textbf{0.84} & \textbf{0.68} & \textbf{0.78} & \textbf{0.64} & \textbf{0.84} & \textbf{0.72} \\[1.5pt]
\bottomrule[1.5pt]
\end{tabular}}
\end{center}
\vspace{-0.8cm}
\end{minipage}
\end{adjustbox}
\hfill
\begin{adjustbox}{valign=t}
\centering
\begin{minipage}{0.49\textwidth}
\centering
\setlength{\tabcolsep}{2pt}
\caption{\textbf{Performance on AgeDB-IT2I based on GPT-4o evaluation.} The scores are from 0 to 10, with higher scores indicating the individual resembles the well-known person.}
\vspace{-0.6cm}
\label{table:total-gpt} 
\small
\begin{center}
\resizebox{\textwidth}{!}{
\begin{tabular}{l|cc|cc|cc|cc}
\toprule[1.5pt]
Datasets (-IT2I) & \multicolumn{4}{c|}{AgeDB} & \multicolumn{2}{c|}{{VGGFace}} & \multicolumn{2}{c}{{CIFAR-100}} \\ \midrule
Size & \multicolumn{2}{c|}{Small} & \multicolumn{2}{c|}{Large} & \multicolumn{2}{c|}{Large} & \multicolumn{2}{c}{Large} \\ \midrule
Metric & \multicolumn{8}{c}{Human Score~$\uparrow$} \\ \midrule
Shot & All & Few & All & Few & All & Few & All & Few \\ \midrule
\textsc{Vanilla} & 5.20 & 3.20 & 4.90 & 3.60 & 4.50 & 2.90 & 6.00 & 3.20 \\[1.5pt]
\textsc{CBDM} & 4.50 & 1.10 & 3.10 & 1.70 & 2.80 & 1.30 & 3.40 & 2.00 \\[1.5pt]
{\textsc{T2H}} & 5.50 & 3.10 & 4.70 & 3.90 & 4.60 & 3.10 & 6.20 & 3.60 \\[1.5pt]
\textsc{PoGDiff (Ours)} & \textbf{9.10} & \textbf{8.40} & \textbf{8.50} & \textbf{8.00} & \textbf{8.20} & \textbf{7.60} & \textbf{8.40} & \textbf{8.00} \\[1.5pt]
\bottomrule[1.5pt]
\end{tabular}}
\end{center}
\vspace{-0.8cm}
\end{minipage}
\end{adjustbox}
\end{table} 

{\textbf{Baselines.} 
We employ \textbf{Stable Diffusion v1.5}~\citep{rombach2022high} as the backbone diffusion model. 
As this is the first work to explore imbalanced text-to-image (IT2I) diffusion models with \textbf{natural text prompts}, we adapt the current state-of-the-art methods designed for long-tailed T2I diffusion models \textbf{with one-hot text prompts} to serve as baselines. The baselines are described below:
\begin{itemize}[nosep,leftmargin=15pt]
\item \textbf{Vanilla:} The \emph{Vanilla} model simply fine-tunes a Stable Diffusion model without additional modifications.

\item \textbf{CBDM:} {\emph{CBDM}~\citep{qin2023class} is a Class Balancing Diffusion Model that incorporates a distribution adjustment regularizer during training. 
During fine-tuning,} we sample an additional text embedding $\y'$ from the entire fine-tuning dataset and apply the CBDM objective function. All hyperparameters are kept the same as in the original paper, with further details available in~\citet{qin2023class}. 

\item {\textbf{T2H:} \emph{T2H}~\citep{zhang2024long} is a Long-Tailed Diffusion Model with Oriented Calibration. 
It is a feature augmentation method, but is not directly applicable to our setting. 
Specifically, T2H relies on the class frequency, which is not available in our experiments. We adapted this method to our settings by using the density for each text prompt embedding to serve as the class frequency in T2H~\citep{zhang2024long}.}
\end{itemize}}

\textbf{Evaluation Protocols and Metrics.}
We use {three types of evaluation metrics: \textbf{generation diversity}}, \textbf{generation accuracy}, and \textbf{generation quality}. 

{For \textbf{generation diversity}, we propose a new metric, generative recall (gRecall), which evaluates the generative diversity of a
model when generation accuracy is strictly enforced. 
\begin{itemize}[nosep,leftmargin=15pt]
\item \textbf{gRecall in the Context of Image Generation: ``Correct Image'' and ``Covered Image''.} For each generated image, we classify it as a ``correct image'' if its distance to at least one ground-truth (GT) image is below a predefined threshold. For instance, suppose we have two training-set images for Einstein, denoted as $x_1$ and $x_2$. A generated image $x_g$ is a ``correct image'' if the cosine similarity between $x_g$ and either $x_1$ or $x_2$ is above some threshold (e.g., we set to $0.7$ here). For example, if the cosine similarity $x_g$ and $x_1$ is larger than $0.7$, we say that $x_g$ is a ``correct image'', and that $x_1$ is a ``covered image''. Intuitively, a training-set image (e.g., $x_1$) is covered if a diffusion model is capable of generating a similar image.
    
\item \textbf{Cosine Similarity between Images.} Note that in practice, we compute the cosine similarity between DINO embeddings of images rather than raw pixels.

\item \textbf{Formal Definition for gRecall.} Formally, for each model, we compute the \textbf{gRecall} per ID as follows:
    \begin{align*}
        {\text{\small gRecall}} = \frac{1}{c} \sum_{i=1}^{c} \frac{ \text{ \small $\#$ of unique covered images for ID i}}{\text{\small $\#$ of images for ID i in the training set}}
    \end{align*}
    where $c$ is the number of IDs in a training set. 

\item \textbf{Analysis.} This metric evaluates the generational diversity of a model. 
For example, if the training dataset contains two distinct images of Einstein, $x_1$ and $x_2$, and a model generates only images resembling $x_1$, the gRecall, in this case, would be $0.5$. 
While the model may achieve high accuracy in terms of facial identity (~\tabref{table:total-human} and ~\tabref{table:total-gpt}), it falls short in diversity because it fails to generate images resembling $x_2$. 
In contrast, if a model generates images that cover both $x_1$ and $x_2$, the gRecall for this ID will be $1$; for instance, if the model generates 10 images for Einstein, where 6 of them resemble $x_1$, and 4 of them resemble $x_2$, the gRecall would be $1$, indicating high diversity and coverage.
\end{itemize}

To assess \textbf{generation accuracy}, we use 10 different seeds to sample 10 images for each minority class. We then gather feedback from both the GPT-4o model~\citep{achiam2023gpt} and human evaluators to score the accuracy of identity recognition. 
Additionally, we employ a pre-trained DINO model~\citep{caron2021emerging} for calculating the DINO score for image similarities. More details about the evaluation process, including prompts we used, are in ~\appref{app:eval}.

\begin{figure*}[t]
\centering
\vskip -0.3cm
\includegraphics[width=0.99\textwidth]{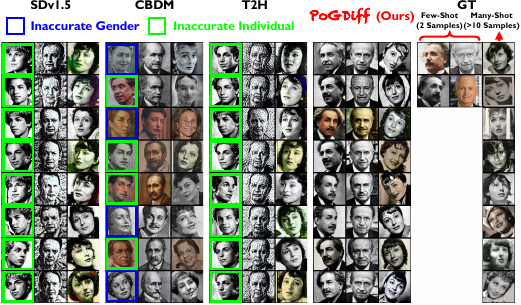}
\vskip -0.3cm
\caption{{Example generated images from different methods. Our PoGDiff outperforms the baselines in both generation accuracy and quality. Regarding the ground truth (GT), the training set for the minority class (left two columns) contains only $2$ images per individual while the majority class has more than $8$ samples per individual.}}
\vskip -0.4cm
\label{fig:pogdiff_img_comp}
\end{figure*}

{For general text-to-image \textbf{generation quality}, we report the widely used Fréchet Inception Distance (FID) score~\citep{heusel2017gans}. 
For all the facial datasets, we use a pre-trained face recognition model as the feature extractor rather than traditional Inception-v3~\citep{szegedy2016rethinking}; since our goal is to evaluate the ability to recognize humans, we need to capture facial features rather than general features.
For CIFAR-100-LT, we employ the original feature extractor (i.e., Inception-v3~\citep{szegedy2016rethinking}). More details are in~\appref{app:eval}.}

\subsection{Results}  

\textbf{Generation Quality and Accuracy.} 
We report the performance of different methods in terms of FID score, human evaluation score, GPT-4o score, and DINO score in~\tabref{table:total-fid}, ~\tabref{table:total-dino}, ~\tabref{table:total-human} and ~\tabref{table:total-gpt}, respectively\footnote{{CLIP score is not applicable here. Our text prompts are predominantly human names, while CLIP is primarily trained on common objects, not human names; therefore, CLIP score cannot measure the matching between images and human names. }}. 
{Full results are included in~\appref{app:res}.} 
Across all tables, we observe that our PoGDiff consistently outperforms all baselines. 
More results and discussions are in ~\appref{app:fid_limit}. 
Notably, PoGDiff demonstrates significant improvements, especially in few-shot scenarios (i.e., for minority classes). 
It is also worth noting that CBDM~\citep{qin2023class} performs extremely poorly on AgeDB-IT2I-S dataset. 
This is because their method samples text conditions from the entire space, which may work in one-hot class settings, but in our context (natural text conditions), this sampling technique misguides the model during training. 

{\figref{fig:pogdiff_img_comp} shows randomly sampled generated images on low-density classes (Column $1 \& 2$) and high-density class (Column $3$) in AgeDB-IT2I-L for each method. 
Note that the ground-truth (GT) images are the training images. 
For the high-density class, we select $8$ out of $24$ total images in the training set to report in this figure. 
Across each column, the individual names are Albert Einstein, JW Marriott, and Luise Rainer, respectively. 
PoGDiff achieves significantly better accuracy and quality for both head and tail classes (see~\appref{app:diss} for more comparisons and analysis). 
Specifically, both SDv1.5 and T2H fail to generate accurate individuals (\green{green} boxes), and CBDM even struggles to generate images with correct genders (\blue{blue} boxes). 
By contrast, our PoGDiff successfully generates accurate individuals, even when trained on a dataset containing only two images.}

\textbf{Generation Diversity.} \tabref{table:total-fid} and \figref{fig:pogdiff_img_comp} demonstrate our PoGDiff's promising generation diversity:
\begin{itemize}[nosep,leftmargin=12pt]
    
\item \textbf{PoGDiff's Superior FID Performance.} \tabref{table:total-fid} shows that PoGDiff achieves a lower FID score, particularly in few-shot regions (i.e., minorities). 
This suggests that the images generated by our method capture a broader range of variations present in the training dataset, such as \textbf{backgrounds or facial angles}.

\item \textbf{PoGDiff's Qualitative Results.} As shown in~\figref{fig:pogdiff_img_comp}:
\begin{itemize}[nosep,leftmargin=8pt]
    \item For Einstein (Column 1 for each method), the training dataset (the GT section on the right) includes two face angles and two hairstyles. Our generated results successfully cover these attributes.
    \item For JW Marriott (Column 2 for each method), the training dataset has only one face angle. Correspondingly our results focus on generating subtle variations in facial expressions with only one angle, \textbf{as expected}.
    \item For the majority group (Column 3 for each method), our PoGDiff's generated images cover a wider range of diversity while maintaining ID consistency.
\end{itemize}
\end{itemize}
\begin{wraptable}{R}{0.6\textwidth}
\setlength{\tabcolsep}{2.5pt}
\vspace{-0.6cm}
\caption{{\textbf{Performance based on gRecall score.} See the detailed definition of gRecall in~\secref{sec:setup}.}}
\vspace{-0.8pt}
\label{table:total-recall}   
\small
\begin{center}
\resizebox{0.99\linewidth}{!}{
\begin{tabular}{l|cc|cc|cc|cc}
\toprule[1.5pt]
Datasets (-IT2I) & \multicolumn{4}{c|}{AgeDB} & \multicolumn{2}{c|}{{VGGFace}} & \multicolumn{2}{c}{{CIFAR-100}} \\ \midrule
Size & \multicolumn{2}{c|}{Small} & \multicolumn{2}{c|}{Large} & \multicolumn{2}{c|}{Large} & \multicolumn{2}{c}{Large} \\ \midrule
Metric & \multicolumn{8}{c}{gRecall Score~$\uparrow$} \\ \midrule
Shot & All & Few & All & Few & All & Few & All & Few \\ \midrule
\textsc{Vanilla} & 0.017 & 0.000 & 0.196 & 0.200 & 0.133 & 0.167 & 0.200 & 0.160 \\[1.5pt]
\textsc{CBDM} & 0.267 & 0.000 & 0.138 & 0.100 & 0.120 & 0.100 & 0.086 & 0.067 \\[1.5pt]
{\textsc{T2H}} & 0.017 & 0.000 & 0.196 & 0.200 & 0.133 & 0.167 & 0.200 & 0.160 \\[1.5pt]
\textsc{PoGDiff (Ours)} & \textbf{0.800} & \textbf{1.000} & \textbf{0.435} & \textbf{0.540} & \textbf{0.400} & \textbf{0.533} & \textbf{0.433} & \textbf{0.567} \\[1.5pt]
\bottomrule[1.5pt]
\end{tabular}}
\end{center}
\vspace{-0.6cm}
\end{wraptable}
%

\textbf{Results on gRecall.} 
\tabref{table:total-recall} shows the gRecall for different methods on four datasets, AgeDB-IT2I-small, AgeDB-IT2I-large, VGGFace-IT2I, and CIFAR-100-IT2I. These results show that our PoGDiff achieves much higher gRecall compared to all baselines, demonstrating its impressive diversity and coverage of different attributes of the same individual in the training set (see~\appref{app:recall_example} for more discussion and examples on gRecall).

\section{Conclusions}
In this paper, we propose a general fine-tuning approach called PoGDiff to address the performance drop that occurs when fine-tuning on imbalanced datasets. 
Instead of directly minimizing the KL divergence between the predicted and ground-truth distributions, PoGDiff replaces the ground-truth distribution with a Product of Gaussians (PoG), constructed by combining the original ground-truth targets with the predicted distribution conditioned on a neighboring text embedding. 
Looking ahead, an interesting avenue for future research would be to explore more innovative techniques for re-weighting minority classes (as limitations discussed in~\secref{app:limitation}), particularly within the constraints of (1) long-tailed generation settings, as opposed to recognition tasks, and (2) natural text prompts rather than one-hot class labels. {Exploring PoGDiff for other modalities (e.g., videos and time series) is also an interesting future work.}


\newpage
\bibliography{reference}
\bibliographystyle{unsrtnat}

\newpage
\appendix
\onecolumn
\section{Proofs for~\prpref{prp:objective}} \label{app:proof}
\begin{proposition}
Assume $\lambda_{\y} = \lambda_{\textrm{PoG}} \triangleq \lambda_t + \lambda_{\y'}$, we have our loss function
\begin{align}
    \mathcal{L}_{t-1}^{\textrm{PoGDiff}} = \mathbb{E}_{q} \left[ 
    \frac{\lambda_{\y}}{2} \left\| \mu_{\theta} (\x_{t}, \y) - \mu_{\textrm{PoG}} \right\|^2 \right] + C.
\end{align}
Here, $C$ is a constant, and $\mu_{\textrm{PoG}}$ denotes the mean of the PoG, with the expression defined in~\eqnref{eq:pog_def}. Then, through derivations based on Gaussian properties, we obtain
\begin{align}
    \mathcal{L}_{t-1}^{\textrm{PoGDiff}} 
    &\leq \mathbb{E}_{q} \left[ 
    \mathcal{A} (\lambda_t) \left\| \epsilon_{\theta} (\x_{t}, \y) - \epsilon \right\|^2 + \mathcal{A} (\lambda_{\y'}) \left\| \epsilon_{\theta} (\x_{t}, \y) - \epsilon_{\theta} (\x_{t}, \y') \right\|^2 \right] + C 
\end{align}
where the function $\mathcal{A} (\lambda) \triangleq \frac{\lambda (1 - \alpha_t)^2}{2 \alpha_t (1 - \Bar{\alpha}_t)}$.
\end{proposition}
\begin{proof}
To prove the above inequality, we need to prove the following lemma. 

\begin{lemma} \label{lm:pog_ineq}
Assume $\lambda_{\y} = \lambda_{\textrm{PoG}} \triangleq \lambda_t + \lambda_{\y'}$, and for simplicity we shorten the notation from $\epsilon_{\theta} (\x_{t}, \y)$ and $\mu_{\theta} (\x_{t}, \y)$ to $\epsilon_{\theta} (\y)$ and $\mu_{\theta} (\y)$, respectively. Then we have 
\begin{align}
    \frac{1}{2} \lambda_t \left( \mu_{\theta} \left( \y \right) - \mu_t \right)^2 + \frac{1}{2} \lambda_{\y'} \left( \mu_{\theta} \left( \y \right) - \mu_{\theta} \left( \y' \right) \right)^2 \geq \frac{1}{2} \lambda_{\y} \left( \mu_{\theta} \left( \y \right) - \mu_{\textrm{PoG}} \right)^2
\end{align}
\end{lemma}
\begin{proof}
By the definition of Gaussian property, we have
\begin{align}
    &\quad \frac{1}{2} \lambda_t \left( \mu_{\theta} (\y) - \mu_t \right)^2 + \frac{1}{2} \lambda_{\y'} \left( \mu_{\theta} (\y)  - \mu_{\theta} (\y') \right)^2 \nonumber \\
    &= \frac{[\mu_{\theta} (\y)]^2 - 2 \mu_t \mu_{\theta} (\y) + \mu_{t}^2}{2 \lambda_{t}^{-1}} + \frac{[\mu_{\theta} (\y)]^2 - 2 \mu_{\theta} (\y') \mu_{\theta} (\y) + [\mu_{\theta} (\y')]^2}{2 \lambda_{\y'}^{-1}} \nonumber \\
    &= \frac{ \left( \lambda_{t}^{-1} + \lambda_{\y'}^{-1} \right) [\mu_{\theta} (\y)]^2 - 2 \left( \frac{\mu_{t}}{\lambda_{\y'}} + \frac{\mu_{\theta} (\y')}{\lambda_{t}} \right) \mu_{\theta} (\y) + \frac{\mu_{t}^{2}}{\lambda_{\y'}} + \frac{[\mu_{\theta} (\y')]^2}{\lambda_t}}{2 [\lambda_t \lambda_{\y'}]^{-1}} \nonumber \\
    &= \frac{[\mu_{\theta} (\y)]^2 - 2 \left( \frac{\mu_{t} \lambda_t + [\mu_{\theta} (\y')] \lambda_{\y'}}{\lambda_{t} + \lambda_{\y'}} \right) \mu_{\theta} (\y) + \frac{\mu_{t}^{2} \lambda_t + [\mu_{\theta} (\y')]^2 \lambda_{\y'}}{\lambda_{t} + \lambda_{\y'}}}{\frac{2}{\lambda_{t} + \lambda_{\y'}}} \nonumber \\
    &\quad + \frac{\left[\frac{\mu_{t} \lambda_t + [\mu_{\theta} (\y')] \lambda_{\y'}}{\lambda_{t} + \lambda_{\y'}}\right]^2}{\frac{2}{\lambda_{t} + \lambda_{\y'}}} - \frac{\left[\frac{\mu_{t} \lambda_t + [\mu_{\theta} (\y')] \lambda_{\y'}}{\lambda_{t} + \lambda_{\y'}}\right]^2}{\frac{2}{\lambda_{t} + \lambda_{\y'}}} \nonumber \\
    &= \frac{\left( \mu_{\theta} (\y) - \frac{\mu_{t} \lambda_t + [\mu_{\theta} (\y')] \lambda_{\y'}}{\lambda_{t} + \lambda_{\y'}} \right)^2 }{\frac{2}{\lambda_{t} + \lambda_{\y'}}} + \frac{(\mu_{t}^{2} \lambda_t + [\mu_{\theta} (\y')]^2 \lambda_{\y'})(\lambda_{t} + \lambda_{\y'}) - (\mu_{t} \lambda_t + [\mu_{\theta} (\y')] \lambda_{\y'})^2}{2 (\lambda_{t} + \lambda_{\y'})} \nonumber \\
    &= \frac{1}{2} \lambda_{\y} \left( \mu_{\theta} (\y) - \mu_{\textrm{PoG}} \right)^2 + \frac{\lambda_{t} \lambda_{\y'} (\mu_t - \mu_{\theta} (\y'))^2}{2 (\lambda_{t} + \lambda_{\y'})} \nonumber \\
    &\geq \frac{1}{2} \lambda_{\y} \left( \mu_{\theta} (\y) - \mu_{\textrm{PoG}} \right)^2 \nonumber.
\end{align}
Thus we complete the proof.
\end{proof}
From~\lemref{lm:pog_ineq}, we can derive
\begin{eqnarray*}
    \frac{1}{2} \lambda_{\y} \left\| \mu_{\theta} (\y) - \mu_{\textrm{PoG}} \right\|^2 &\equiv& \frac{1}{2} \lambda_{\y} \left( \mu_{\theta} (\y) - \mu_{\textrm{PoG}} \right)^2 \nonumber \\
    &\leq& \frac{1}{2} \lambda_t \left( \mu_{\theta} (\y) - \mu_t \right)^2 + \frac{1}{2} \lambda_{\y'} \left( \mu_{\theta} (\y) - \mu_{\theta} (\y') \right)^2 \nonumber \\
    &\equiv& \frac{1}{2} \lambda_t \left\| \mu_{\theta} (\y) - \mu_t \right\|^2 + \frac{1}{2} \lambda_{\y'} \left\| \mu_{\theta} (\y) - \mu_{\theta} (\y') \right\|^2 \nonumber \\
    &\equiv& \mathcal{A} (\lambda_t) \left\| \epsilon_{\theta} (\y) - \epsilon \right\|^2 + \mathcal{A} (\lambda_{\y'}) \left\| \epsilon_{\theta} (\y) - \epsilon_{\theta} (\y') \right\|^2, \nonumber
\end{eqnarray*}
where the function $\mathcal{A} (\lambda) \triangleq \frac{\lambda (1 - \alpha_t)^2}{2 \alpha_t (1 - \Bar{\alpha}_t)}$, and the last equivalence is because the transform from $\mu_{\theta} (\cdot)$ to $\epsilon_{\theta} (\cdot)$.
\end{proof}

\section{Ablation Study} \label{app:ablation}
\begin{wraptable}{R}{0.6\textwidth}
\setlength{\tabcolsep}{2.5pt}
\vspace{-8pt}
\caption{\textbf{Ablation Studies.}}
\vspace{-0.8pt}
\label{table:ablation}
\small
\begin{center}
\resizebox{0.59\textwidth}{!}{
\begin{tabular}{l|cc|cc|cc|cc}
\toprule[1.5pt]
Datasets & \multicolumn{8}{c}{AgeDB-IT2I-Large} \\ \midrule
Size & \multicolumn{2}{c|}{FID~$\downarrow$} & \multicolumn{2}{c|}{Human ~$\uparrow$} & \multicolumn{2}{c|}{GPT-4o~$\uparrow$} & \multicolumn{2}{c}{DINO~$\uparrow$}  \\ \midrule
Shot & All & Few & All & Few & All & Few & All & Few \\ \midrule
\textsc{w/o $\y'$ (Vanilla)} & 7.67 & 11.67 & 0.60 & 0.20 & 4.90 & 3.60 & 0.34 & 0.25 \\[1.5pt]
\textsc{w/o $\psi_\textrm{img-sim}$} & 6.41 & 10.49 & \textbf{0.84} & \textbf{0.68} & 8.40 & 7.60 & 0.57 & 0.46 \\[1.5pt]
\textsc{w/o $\psi_{\textrm{inv-txt-den}}$} & 6.35 & 10.43 & \textbf{0.84} & \textbf{0.68} & 8.20 & 7.80 & 0.64 & 0.51 \\[1.5pt]
\textsc{PoGDiff (Ours)} & \textbf{6.03} & \textbf{10.16} & \textbf{0.84} & \textbf{0.68} & \textbf{8.50} & \textbf{8.00} & \textbf{0.66} & \textbf{0.52} \\ 
\bottomrule[1.5pt]
\end{tabular}}
\end{center}
\vspace{-0.4cm}
\end{wraptable}
To verify the effectiveness of each component in the second term in our PoGDiff final objective function from~\eqnref{eq:psi_final}, we report the accuracy of our proposed PoGDiff after removing the $\y'$ (i.e., same as Vanilla model), the Image Similarity term $\psi_\textrm{img-sim}$, and/or the Inverse Text Densities term $\psi_\textrm{inv-txt-den}$ in~\tabref{table:ablation} for AgeDB-IT2I-L. The results show that removing either term leads to a performance drop, confirming the importance of both terms in our PoGDiff.

\section{Details for Datasets} \label{app:datasets}
Note that our method is designed for fine-tuning. Therefore our setup does not require large-scale, long-tailed human datasets. Instead, we sample from these datasets, as long as they meet the following criteria: (1) the dataset must be long-tailed, (2) traditional methods must fail to recognize the minority classes, and (3) there must be a distinguishable difference between the majority and minority classes (e.g., we prefer visual distinctions between the two groups to better highlight the impact of our method). \figref{fig:density_per_exp_bar} shows the label density distribution of these datasets, and their level of imbalance (see~\appref{app:sparse} for details on data sparsity).  

\emph{AgeDB-IT2I:} AgeDB-IT2I is constructed from the AgeDB dataset~\citep{moschoglou2017agedb}. For each image $\x$ in AgeDB, we passed it through the pretrained LLaVA-1.6-7b model~\citep{liu2024visual} to generate textual captions $\tilde{\y}$. Since the identities in AgeDB are well-known individuals that the pretrained SDv1.5~\citep{rombach2022high} might have encountered during pre-training, we masked the true names and replaced them with generic, random names, leading to a new caption $\y$. 
For example, we replace ``Albert Einstein'' in the caption with a random name ``Lukas''. Finally, we collect all $(\y,\x)$ pairs to form our AgeDB-IT2I dataset. 

Additionally, given that the identities (i.e., people or individuals) in AgeDB are well-known figures, we sampled from AgeDB to create three datasets for comprehensive analysis: AgeDB-IT2I-L (large), AgeDB-IT2I-M (medium), and AgeDB-IT2I-S (small). Specifically:
\begin{itemize}[nosep,leftmargin=15pt]
    \item \emph{AgeDB-IT2I-L (large).} This dataset consists of $976$ images across $223$ identities, with each majority class containing $30$ images and each minority class containing $2$ images.
    \item \emph{AgeDB-IT2I-M (medium).}  This dataset consists of $100$ images across $10$ identities, with each majority class containing $30$ images and each minority class containing $2$ images.
    \item \emph{AgeDB-IT2I-S (small).} This dataset contains $32$ images across $2$ identities, where each majority class consists of $30$ images and each minority class consists of $2$ images.
\end{itemize}

\emph{DigiFace-IT2I:} DigiFace-IT2I is derived from the DigiFace dataset~\citep{bae2023digiface}. It contains $985$ images across $179$ identities, where each majority class consists of $30$ images and each minority class consists of $2$ images. 
We use a process similar to AgeDB-IT2I to collect text-image pairs, forming this DigiFace-IT2I dataset. 

{\emph{VGGFace-IT2I:} VGGFace-IT2I is a subset from VGGFace2~\citep{cao2018vggface2}. It contains $1933$ images across $193$ identities, where each majority class consists of $49$ images and each minority class consists of $2$ images. }

\section{Details for Baselines} \label{app:baseline}
We employ \textbf{Stable Diffusion v1.5}~\citep{rombach2022high} as the backbone diffusion model. As this is the first work to explore imbalanced text-to-image (IT2I) diffusion models with \textbf{natural text prompts}, we adapt the current state-of-the-art methods designed for long-tailed T2I diffusion models \textbf{with one-hot text prompts} to serve as baselines. The baselines are described below:
\begin{itemize}[nosep,leftmargin=15pt]
    \item \emph{Vanilla:} We use term \textbf{Vanilla} to denote a model that does not incorporate any techniques for handling imbalanced data, equivalent to fine-tuning a Stable Diffusion model without additional modifications.
    \item \emph{CBDM:} We use term \textbf{CBDM} to denote a model that incorporates the Class Balancing Diffusion Model (CBDM)~\citep{qin2023class} approach. During fine-tuning, we sample an additional text embedding $\y'$ from the entire fine-tuning dataset and apply the CBDM objective function. All hyperparameters are kept the same as in the original paper, with further details available in~\citet{qin2023class}. 
    \item {\emph{T2H:} We use the term \textbf{T2H} to denote a model that uses Long-Tailed Diffusion Models with Oriented Calibration (T2H)~\citep{zhang2024long}. T2H is a feature augmentation method, but is not directly applicable to our setting. Specifically, T2H~\citep{zhang2024long} relies on the class frequency, which is not available in our experiments. In this paper, we adapt this method to our settings by using the density for each text prompt embedding to serve as the class frequency in T2H~\citep{zhang2024long}.}
\end{itemize}

\section{Details for Evaluation} \label{app:eval}
In this section, we provide details on our evaluation procedures.

\textbf{FID Score.} For each identity, we collect all images from the original AgeDB or DigiFace datasets as the \emph{true image set}. Then, In \emph{all-shot} evaluation, for AgeDB-IT2I-S and AgeDB-IT2I-M, we generate $100$ images per identity as the \emph{fake image set}, and for AgeDB-IT2I-L, DigiFace-IT2I, VGGFace-IT2L and CIFAR-100-IT2I, we generate $20$ images per identity as the \emph{fake image set}. In \emph{few-shot} evaluation, we generate $500$ images per identity as the \emph{fake image set}. For all generations, we employ the DDIM sampling technique~\citep{song2020denoising} with $50$ steps. The prompt used during generation is ``An image of \{p\}.'' where ``p'' is the name of the identity (e.g., Albert Einstein).

\textbf{Human \& GPT-4o Feedback.} For each minority identity, we generate $5$ images using DDIM sampling~\citep{song2020denoising} with $50$ steps. We then ask $10$ people to evaluate whether the images depict the same person (scored as $1.0$) or not (scored as $0.0$). Additionally, for each image, we ask the GPT-4o model to rate the similarity on a scale from $1$ to $10$. The prompt used during generation is ``An image of \{ p \}.'' where ``p'' is the name of the identity. The text prompt using for GPT-4o model is ``It is mandatory to give a score that how close the person in the image to a well-known individual. A score of 10.0 means they are exactly the same person, while a score of 0.0 means they are definitely not the same person. How close you think the person in the image is to `p-true'.'' where ``p-true'' denotes the real name (well-known name) in AgeDB. Note that the GPT-4o model might occasionally refuse to provide a score, and you may need to repeat and compel it to give a rating. For each image, we collect 10 scores from the GPT-4o model and report the average rating.

\textbf{Evaluating Image Similarities.} We collect samples that are outside our training dataset (e.g., AgeDB-T2I-L) but belong to the original dataset (e.g., AgeDB). Using the same prompt, we generate the corresponding images. A pre-trained DINOv2 model~\citep{caron2021emerging} is then applied to extract latent features, and cosine similarities are calculated.

\section{Details for Implementation}
For both baselines and our model, we used the same hyper-parameter settings, specifically
\begin{itemize}[nosep,leftmargin=15pt]
    \item \emph{AgeDB-IT2I-M $\&$ AgeDB-IT2I-S.} The learning rate was set to $1 \times 10^{-5}$, with a maximum of $6,000$ training steps. The effective batch size per GPU was $8$, calculated as $8 \textrm{ (Batch Size)} \times 1 \textrm{ (Gradient Accumulation Step)}$.
    \item \emph{AgeDB-IT2I-L $\&$ DigiFace-IT2I.} The learning rate was set to $1 \times 10^{-5}$, with a maximum of $12,000$ training steps. The effective batch size per GPU was $32$, calculated as $8 \textrm{ (Batch Size)} \times 4 \textrm{ (Gradient Accumulation Steps)}$.
    \item \emph{VGGFace-IT2I $\&$ CIFAR-100-IT2I.} The learning rate was set to $1 \times 10^{-4}$, with a maximum of $24,000$ training steps. The effective batch size per GPU was $128$, calculated as $32 \textrm{ (Batch Size)} \times 4 \textrm{ (Gradient Accumulation Steps)}$.
\end{itemize}

\section{Full Results} \label{app:res}
{We report the performance of different methods in terms of FID score, human evaluation score, GPT-4o score, and DINO score in~\tabref{table:app-fid}, ~\tabref{table:app-dino}, ~\tabref{table:app-human} and ~\tabref{table:app-gpt}, respectively. Across all tables, we observe that our PoGDiff consistently outperforms all baselines. More results and discussions are in ~\appref{app:fid_limit}. Notably, PoGDiff demonstrates significant improvements, especially in few-shot scenarios (i.e., for minority classes). It is also worth noting that CBDM~\citep{qin2023class} performs extremely poorly on AgeDB-IT2I-S and AgeDB-IT2I-M datasets. This is because their method samples text conditions from the entire space, which may work in one-hot class settings, but in our context (natural text conditions), this sampling technique misguides the model during training. In addition, \tabref{table:app-recall} shows the gRecall for different methods on three datasets, AgeDB-IT2I-Small, AgeDB-IT2I-Medium, and AgeDB-IT2I-Large. These results show that our PoGDiff achieves much higher gRecall compared to all baselines, demonstrating its impressive diversity and coverage of different attributes of the same individual in the training set (see~\appref{app:recall_example} for more discussion and examples on gRecall).}
\begin{table}[t]
\centering
\setlength{\tabcolsep}{2pt}
\vspace{-0.2cm}
\caption{\textbf{Performance based on FID score.}}
\vspace{-0.2cm}
\label{table:app-fid}
\scriptsize
\begin{center}
\resizebox{0.99\textwidth}{!}{
\begin{tabular}{l|cc|cc|cc|cc|cc|cc}
\toprule[1.5pt]
Datasets (-IT2I) & \multicolumn{6}{c|}{AgeDB} & \multicolumn{2}{c|}{DigiFace} & \multicolumn{2}{c|}{{VGGFace}} & \multicolumn{2}{c}{{CIFAR-100}} \\ \midrule
Size & \multicolumn{2}{c|}{Small} & \multicolumn{2}{c|}{Medium} & \multicolumn{2}{c|}{Large} & \multicolumn{2}{c|}{Large} & \multicolumn{2}{c|}{Large} & \multicolumn{2}{c}{Large} \\ \midrule
Metric & \multicolumn{12}{c}{FID~$\downarrow$} \\ \midrule
Shot & All & Few & All & Few & All & Few & All & Few & All & Few & All & Few \\ \midrule
\textsc{Vanilla} & 14.88 & 13.72 & 12.87 & 12.56 & 7.67 & 11.67 & 7.18 & 12.23 & 7.59 & 12.08 & 7.19 & 11.46 \\[1.5pt]
\textsc{CBDM} & 14.72 & 14.13 & 11.63 & 11.59 & 7.18 & 11.12 & 6.96 & 12.72 & 7.23 & 11.91 & 7.26 & 11.94 \\[1.5pt]
{\textsc{T2H}} & 14.85 & 13.66 & 12.79 & 12.52 & 7.61 & 11.64 & 7.14 & 12.22 & 7.34 & 12.02 & 7.10 & 11.39 \\[1.5pt]
\textsc{PoGDiff (Ours)} & \textbf{14.15} & \textbf{12.88} & \textbf{10.89} & \textbf{10.64} & \textbf{6.03} & \textbf{10.16} & \textbf{6.84} & \textbf{11.21} & \textbf{6.29} & \textbf{10.97} & \textbf{6.24} & \textbf{9.41} \\[1.5pt]
\bottomrule[1.5pt]
\end{tabular}}
\end{center}
\vspace{-0.1cm}
\end{table}
\begin{table}[t]
\centering
\setlength{\tabcolsep}{2pt}
\vspace{-0.2cm}
\caption{\textbf{Performance based on DINO score.}}
\vspace{-0.2cm}
\label{table:app-dino} 
\scriptsize
\begin{center}
\resizebox{0.99\linewidth}{!}{
\begin{tabular}{l|cc|cc|cc|cc|cc|cc}
\toprule[1.5pt]
Datasets (-IT2I) & \multicolumn{6}{c|}{AgeDB} & \multicolumn{2}{c|}{DigiFace} & \multicolumn{2}{c|}{{VGGFace}} & \multicolumn{2}{c}{{CIFAR-100}} \\ \midrule
Size & \multicolumn{2}{c|}{Small} & \multicolumn{2}{c|}{Medium} & \multicolumn{2}{c|}{Large} & \multicolumn{2}{c|}{Large} & \multicolumn{2}{c|}{Large} & \multicolumn{2}{c}{Large}  \\ \midrule
Metric & \multicolumn{12}{c}{DINO (cosine similarity) scores~$\uparrow$} \\ \midrule
Shot & All & Few & All & Few & All & Few & All & Few & All & Few & All & Few \\ \midrule
\textsc{Vanilla} & 0.42 & 0.37 & 0.39 & 0.28 & 0.34 & 0.25 & 0.42 & 0.36 & 0.41 & 0.29 & 0.48 & 0.32 \\[1.5pt]
\textsc{CBDM} & 0.54 & 0.09 & 0.38 & 0.11 & 0.41 & 0.26 & 0.34 & 0.16 & 0.46 & 0.22 & 0.52 & 0.28 \\[1.5pt]
{\textsc{T2H}} & 0.43 & 0.39 & 0.42 & 0.29 & 0.37 & 0.26 & 0.44 & 0.36 & 0.42 & 0.28 & 0345 & 0.30 \\[1.5pt]
\textsc{PoGDiff (Ours)} & \textbf{0.77} & \textbf{0.73} & \textbf{0.69} & \textbf{0.56} & \textbf{0.66} & \textbf{0.52} & \textbf{0.64} & \textbf{0.49} & \textbf{0.69} & \textbf{0.55} & \textbf{0.73} & \textbf{0.61} \\[1.5pt]
\bottomrule[1.5pt]
\end{tabular}}
\end{center}
\vspace{-0.8cm}
\end{table}
\begin{table}[t]
\centering
\setlength{\tabcolsep}{2pt}
\vspace{-0.2cm}
\caption{\textbf{Performance based on human evaluation.} The evaluation is a binary decision: Image is either judged as representing the same individual (score 1.0) or not (score 0.0).}
\vspace{-0.2cm}
\label{table:app-human}   
\scriptsize
\begin{center}
\resizebox{0.8\linewidth}{!}{
\begin{tabular}{l|cc|cc|cc|cc|cc}
\toprule[1.5pt]
Datasets (-IT2I) & \multicolumn{6}{c|}{AgeDB} & \multicolumn{2}{c|}{{VGGFace}} & \multicolumn{2}{c}{{CIFAR-100}} \\ \midrule
Size & \multicolumn{2}{c|}{Small} & \multicolumn{2}{c|}{Medium} & \multicolumn{2}{c|}{Large} & \multicolumn{2}{c|}{Large} & \multicolumn{2}{c}{Large} \\ \midrule
Metric & \multicolumn{10}{c}{Human Score~$\uparrow$} \\ \midrule
Shot & All & Few & All & Few & All & Few & All & Few & All & Few \\ \midrule
\textsc{Vanilla} & 0.50 & 0.00 & 0.66 & 0.32 & 0.60 & 0.20 & 0.62 & 0.16 & 0.72 & 0.30 \\[1.5pt]
\textsc{CBDM} & 0.50 & 0.00 & 0.44 & 0.08 & 0.56 & 0.12 & 0.54 & 0.10 & 0.63 & 0.24 \\[1.5pt]
{\textsc{T2H}} & 0.50 & 0.00 & 0.66 & 0.32 & 0.60 & 0.20 & 0.62 & 0.16 & 0.72 & 0.30 \\[1.5pt]
\textsc{PoGDiff (Ours)} & \textbf{1.00} & \textbf{1.00} & \textbf{0.96} & \textbf{0.92} & \textbf{0.84} & \textbf{0.68} & \textbf{0.78} & \textbf{0.64} & \textbf{0.84} & \textbf{0.72} \\[1.5pt]
\bottomrule[1.5pt]
\end{tabular}}
\end{center}
\vspace{-0.3cm}
\end{table}
\begin{table}[t]
\centering
\setlength{\tabcolsep}{2pt}
\vspace{-0.2cm}
\caption{\textbf{Performance based on GPT-4o evaluation.} The scores are from 0 to 10, with higher scores indicating the individual resembles the well-known person.}
\vspace{-0.2cm}
\label{table:app-gpt} 
\scriptsize
\begin{center}
\resizebox{0.8\textwidth}{!}{
\begin{tabular}{l|cc|cc|cc|cc|cc}
\toprule[1.5pt]
Datasets (-IT2I) & \multicolumn{6}{c|}{AgeDB} & \multicolumn{2}{c|}{{VGGFace}} & \multicolumn{2}{c}{{CIFAR-100}} \\ \midrule
Size & \multicolumn{2}{c|}{Small} & \multicolumn{2}{c|}{Medium} & \multicolumn{2}{c|}{Large} & \multicolumn{2}{c|}{Large} & \multicolumn{2}{c}{Large} \\ \midrule
Metric & \multicolumn{10}{c}{GPT-4o Evaluation~$\uparrow$} \\ \midrule
Shot & All & Few & All & Few & All & Few & All & Few & All & Few \\ \midrule
\textsc{Vanilla} & 5.20 & 3.20 & 4.30 & 2.90 & 4.90 & 3.60 & 4.50 & 2.90 & 6.00 & 3.20 \\[1.5pt]
\textsc{CBDM} & 4.50 & 1.10 & 1.30 & 1.00 & 3.10 & 1.70 & 2.80 & 1.30 & 3.40 & 2.00 \\[1.5pt]
{\textsc{T2H}} & 5.50 & 3.10 & 4.60 & 3.00 & 4.70 & 3.90 & 4.60 & 3.10 & 6.20 & 3.60 \\[1.5pt]
\textsc{PoGDiff (Ours)} & \textbf{9.10} & \textbf{8.40} & \textbf{8.80} & \textbf{8.20} & \textbf{8.50} & \textbf{8.00} & \textbf{8.20} & \textbf{7.60} & \textbf{8.40} & \textbf{8.00} \\[1.5pt]
\bottomrule[1.5pt]
\end{tabular}}
\end{center}
\vspace{-0.8cm}
\end{table}
\begin{table}[t]
\centering
\setlength{\tabcolsep}{2pt}
\vspace{-0.2cm}
\caption{{\textbf{Performance based on gRecall score.} See the detailed definition of gRecall in~\secref{sec:setup}.}}
\vspace{-0.8pt}
\label{table:app-recall}   
\small
\begin{center}
\resizebox{0.9\linewidth}{!}{
\begin{tabular}{l|cc|cc|cc|cc|cc}
\toprule[1.5pt]
Datasets (-IT2I) & \multicolumn{6}{c|}{AgeDB} & \multicolumn{2}{c|}{{VGGFace}} & \multicolumn{2}{c}{{CIFAR-100}} \\ \midrule
Size & \multicolumn{2}{c|}{Small} & \multicolumn{2}{c|}{Medium} & \multicolumn{2}{c|}{Large} & \multicolumn{2}{c|}{Large} & \multicolumn{2}{c}{Large} \\ \midrule
Metric & \multicolumn{10}{c}{gRecall Score~$\uparrow$} \\ \midrule
Shot & All & Few & All & Few & All & Few & All & Few & All & Few \\ \midrule
\textsc{Vanilla} & 0.017 & 0.000 & 0.104 & 0.167 & 0.196 & 0.200 & 0.133 & 0.167 & 0.200 & 0.160 \\[1.5pt]
\textsc{CBDM} & 0.267 & 0.000 & 0.159 & 0.083 & 0.138 & 0.100 & 0.120 & 0.100 & 0.086 & 0.067 \\[1.5pt]
{\textsc{T2H}} & 0.017 & 0.000 & 0.104 & 0.167 & 0.196 & 0.200 & 0.133 & 0.167 & 0.200 & 0.160 \\[1.5pt]
\textsc{PoGDiff (Ours)} & \textbf{0.800} & \textbf{1.000} & \textbf{0.517} & \textbf{0.642} & \textbf{0.435} & \textbf{0.540} & \textbf{0.400} & \textbf{0.533} & \textbf{0.433} & \textbf{0.567}\\[1.5pt]
\bottomrule[1.5pt]
\end{tabular}}
\end{center}
\vspace{-0.3cm}
\end{table}

\section{Limitations} \label{app:limitation}

\textbf{Datasets.} Our method relies heavily on ``borrowing'' the statistical strength of neighboring samples from minority classes, making the results sensitive to the size of the minority class. (i.e., in our assumption we require \textbf{at least $2$} for each minority class). {In addition, while our AgeDB-IT2I-small and AgeDB-IT2I-medium are actually the sparse dataset, the cardinality remains limited in our experiments. Therefore, how to address IT2I problem under this settings are interesting directions. } 

\textbf{Models.} Our method is a general fine-tuning approach designed for datasets that the Stable Diffusion (SD) model has not encountered during pre-training. {As shown in~\figref{fig:pogdiff_main}, color deviation is very common and is a known issue when one fine-tunes diffusion models (as also mentioned in~\citep{song2020score}); for example, we can observe similar color deviation in both baselines (e.g., CBDM and Stable Diffusion v1.5) and our PoGDiff. This can be mitigated using the exponential moving average (EMA) technique~\citep{song2020score}; however, this is orthogonal to our method and is outside the scope of our paper. Moreover, as shown in~\figref{fig:pogdiff_img}, the baseline Stable Diffusion also suffers from this issue.} 
Besides, exploring PoGDiff's performance when training from scratch is also an interesting direction for future work.

\textbf{Methodology.} {The distance between the current text embedding $\y$ and the sampled $\y'$ impacts the final generated results, therefore in our paper, we introduced a more sophisticated approach for computing the weight $\psi$, which depends on the quality of the image pre-trained model and our trained VAE. These mechanisms ensure that data points with smaller distances are assigned higher effective weights. Effectively producing $\psi$ for any new, arbitrary dataset remains an open question and is an interesting avenue for future work, as it could further enhance the method’s performance.}

{\textbf{Evaluation.} Our goal is to adapt the pretrained diffusion model to a specific dataset; therefore the evaluation should focus on the target dataset rather than the original dataset used during pre-training. For example, when a user fine-tunes a model on a dataset of employee faces, s/he is not interested in how well the fine-tuned model can generate images of ``tables'' and ``chairs''. Evaluating the model's performance on the original dataset used during pre-training would be an intriguing direction for future work, but it is orthogonal to our proposed PoGDiff and out of the scope of our paper.}

\section{Additional Details for AgeDB-IT2I-small in~\tabref{table:total-recall}} \label{app:recall_example}
{For AgeDB-IT2I-small, there are two IDs, one ``majority'' ID with $30$ images and one minority ID with $2$ images.
\begin{itemize}[nosep,leftmargin=15pt]
    \item For \textbf{VANILLA} and \textbf{T2H}, the gRecall for the majority ID and the minority ID is $1/30$ and $0/2$, respectively. Therefore, the average gRecall score is $0.5 * 1/30 + 0.5 * 0/2 \approx 0.0167$.
    \item For \textbf{CBDM}, the gRecall for the majority ID and the minority ID is $16/30$ and $0/2$, respectively. Therefore, the average gRecall score is $0.5 * 16/30 + 0.5 * 0/2 \approx 0.2667$.
    \item For \textbf{PoGDiff (Ours)}, the gRecall for the majority ID and the minority ID is $18/30$ and $2/2$, respectively. Therefore, the average gRecall score is $0.5 * 18/30 + 0.5 * 2/2 = 0.8$.
\end{itemize}} 

\section{Discussion} \label{app:diss}
\subsection{Problem Settings}
{We would like to clarify that our paper focuses on a setting different from works like DreamBooth~\citep{ruiz2023dreambooth}, and our focus is not on diversity, but on finetuning a diffusion model on an imbalanced dataset. Specifically:
\begin{itemize}
    \item \textbf{Different Setting from Custom Techniques like DreamBooth~\citep{ruiz2023dreambooth}, CustomDiffusion~\citep{kumari2023multi} and PhotoMaker~\citep{li2024photomaker}.} Previous works like CustomDiffusion and PhotoMaker focus on adjusting the model to generate images with \textbf{a single object}, e.g., a specific dog. In contrast, our PoGDiff focuses finetuning the diffusion model on an entire data with \textbf{many different objects/persons simultaneously}. They are \textbf{very different settings} and are \textbf{complementary} to each other. 
    \item \textbf{Diversity.} Note that while our PoG can naturally generate images with diversity, diversity is actually \textbf{not} our focus. Our goal is to fine-tune a diffusion model on an imbalanced dataset. For example, PoGDiff can fine-tune a diffusion model on an imbalanced dataset of employee faces so that the diffusion model can generate new images that match each employee's identity. In this case, we are more interested in ``faithfulness'' rather than ``diversity''.
\end{itemize}}
\begin{figure}[t]
\centering
\vskip -0.3cm
\includegraphics[width=1.0\textwidth]{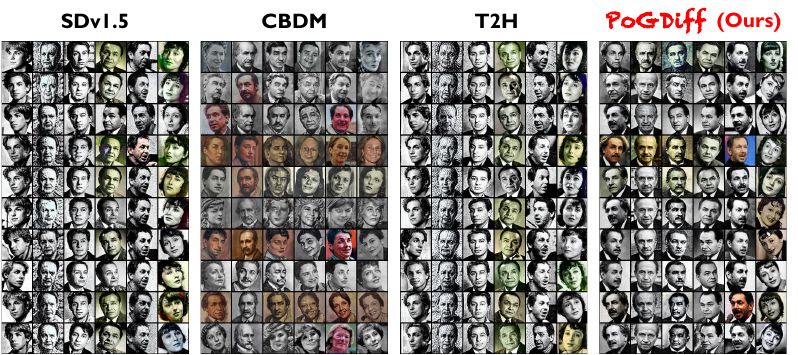}
\vskip -0.3cm
    \caption{{Example generated images from different methods. Our PoGDiff outperforms the baselines in terms of both generation accuracy and generation quality.}}
\vskip -0.1cm
\label{fig:pogdiff_img}
\end{figure}
\begin{figure}[t]
\centering
\vskip -0.1cm
\includegraphics[width=1.0\textwidth]{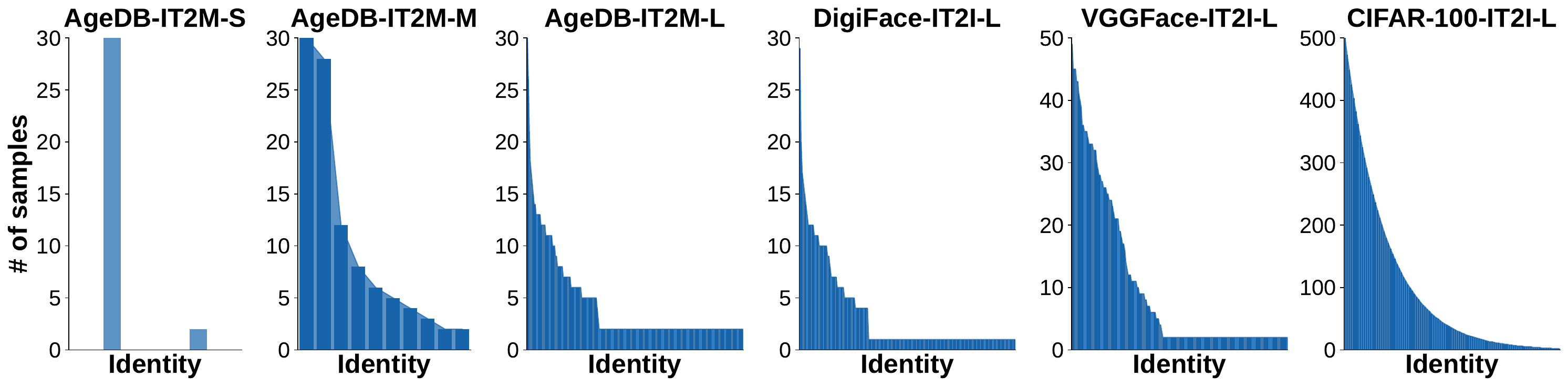}
\vskip -0.3cm
    \caption{{{Overview of label distributions for six IT2I datasets in bar plots. The x-axis corresponds to the identities (i.e., people or class)}.}}
\vskip -0.3cm
\label{fig:density_per_exp_bar}
\end{figure}
%
%
\begin{figure}[t]
\centering
\vskip -0.1cm
\includegraphics[width=0.5\linewidth]{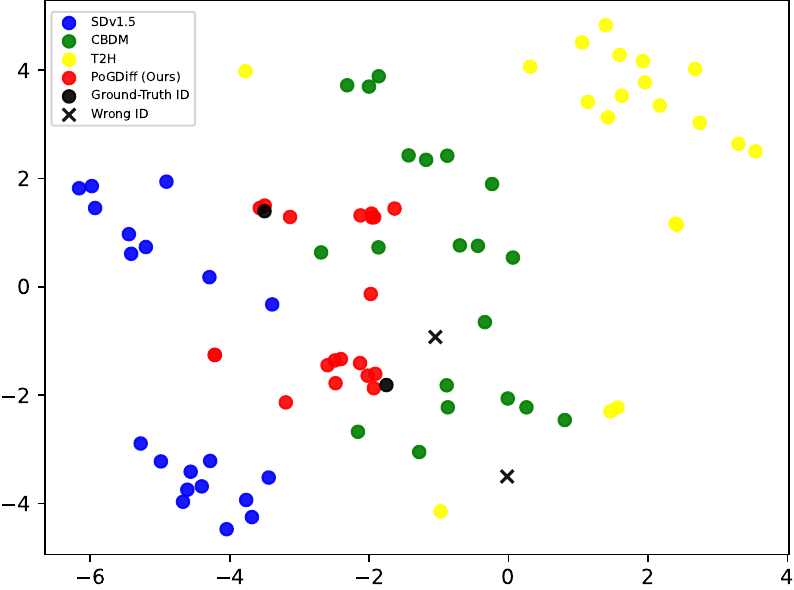}
\vskip -0.3cm
    \caption{{TSNE visualization for all the methods for an example individual in the AgeDB-IT2M-large dataset.}}
\vskip -0.6cm
\label{fig:tsne}
\end{figure}
\subsection{More Detailed Analysis: Understanding~\figref{fig:pogdiff_img}} 
{\figref{fig:pogdiff_img} shows randomly sampled generated images on low-density classes in AgeDB-IT2I-L. Across each column, the individual names are Albert Einstein, JW Marriott, J.P. Morgan, Edward G. Robinson, Larry Ellison, and Luise Rainer, respectively. The results show that our PoGDiff achieves significantly better accuracy and quality for tail classes.}

{Note that one of our primary objectives is to generate accurate images of the same individual while ensuring facial consistency. Therefore \textbf{diversity can sometimes be harmful}. For example, given a text input of ``Einstein'', generated images with high diversity would generate both male and females images; \textbf{this is obviously incorrect}. Therefore it is important to strike a balance between \textbf{diversity} and \textbf{accuracy}, a goal that our PoGDiff achieves. }

{Specifically, as shown in~\figref{fig:pogdiff_img}: 
\begin{itemize}[nosep,leftmargin=15pt]
    \item \textbf{First Three Columns of SDv1.5, CBDM, and PoGDiff}: In these cases, the \textbf{training} dataset contains \textbf{only two images per person}. With such limited data, it is impossible to introduce meaningful diversity.
    \begin{itemize}[nosep,leftmargin=14pt]
        \item SDv1.5 fails to generate accurate images altogether in this scenario.
        \item While CBDM might appear to produce the ``diversity'', it does so incorrectly, as it generates an image of a woman when the target is Einstein.
        \item In contrast, our PoGDiff can successfully generate accurate images (e.g., Einstein images in Column 1) while still enjoying sufficient diversity.
    \end{itemize}
    \item \textbf{Fourth and Fifth Columns}: Here, the training dataset contains a medium number of images per person (5–7 images). Under these conditions:
    \begin{itemize}[nosep,leftmargin=14pt]
        \item SDv1.5 can generate accurate representations of individuals, but its outputs lack diversity.
        \item CBDM, on the other hand, introduces ``diversity'' but consistently generates incorrect results.
        \item In contrast, our method produces accurate images of the target individual while demonstrating greater diversity than SDv1.5.
    \end{itemize}
    \item \textbf{Sixth Column}: In this case, the training dataset includes 30 images per person.
    \begin{itemize}[nosep,leftmargin=14pt]
        \item SDv1.5 generates accurate images but with nearly identical expressions, i.e., poor diversity.
        \item CBDM still fails to generate accurate depictions of the individual.
        \item In contrast, PoGDiff successfully generates accurate images while maintaining diversity.
    \end{itemize}
\end{itemize}
In summary, typical diversity evaluation in diffusion model evaluations, such as generating multiple types of trees for a ``tree'' prompt, is \textbf{not the focus of our setting} and may even be \textbf{misleading}. In our setting, the key is to balance accuracy and diversity. }

\subsection{Why Not Directly Smooth Text Embedding?}
{Preliminary results indicate that directly smoothing the text embeddings does not yield meaningful improvements. Below we provide some insights into why this approach might fail. Suppose we have a text embedding $y$ and its corresponding neighboring embedding $y'$. Depending on their relationship, we are likely to encounter three cases:
\begin{itemize}[nosep,leftmargin=15pt]
    \item \textbf{Case 1: $\y' = \y$.} In this case, applying a reweighting method such as a linear combination results in no meaningful change, as the smoothing outcome is still $\y$. 
    \item \textbf{Case 2: $\y'$ is far from $\y$.} If $\y'$ is significantly distant from $\y$, combining them becomes irrelevant and nonsensical, as $\y'$ no longer represents useful neighboring information.
    \item \textbf{Case 3: $\y'$ is very close to $\y$.} When $\y'$ is close to $\y$, the reweighting can be approximated as: $\alpha \y + (1-\alpha) \y' \approx \y + (1-\alpha)(\y' - \y)$. 
   Since $\y'$ is nearly identical to $\y$, this effectively introduces a small weighted noise term $(1-\alpha)(\y' - \y)$ into $\y$. In our preliminary experiments, this additional noise degraded the performance compared to the original baseline results.
\end{itemize}
Based on these observations, direct smoothing of text embeddings appears ineffective and may even harm performance in some cases.}

\subsection{Our Dataset Covers Different Levels of Sparsity} \label{app:sparse}
{Our AgeDB-IT2M-small and AgeDB-IT2M-medium datasets are actually very sparse and are meant for evaluate the sparse data. For example, the AgeDB-IT2M-small only contains images from 2 persons, it is therefore a very sparse data setting, compared to AgeDB-IT2M-large with images across 223 persons. ~\figref{fig:density_per_exp_bar} shows the bar plot version for our datasets, while sparse settings are not our primary focus, we agree that addressing imbalanced image generation in such setting is an interesting and valuable direction, and we have included a discussion about this in the limitations section of the paper.}

\subsection{Discussion on FID} \label{app:fid_limit}
{It is important to note that the FID score measures only the distance between Gaussian distributions of ground-truth and generated images, relying solely on mean and variance. As a result, it does not fully capture the nuances of our task. This is why we include additional evaluation metrics such as DINO Score, Human Score, and GPT-4o Score, to comprehensively verify our method's superiority (as shown in~\tabref{table:total-dino}, ~\tabref{table:total-human} and~\tabref{table:total-gpt}).}

\textbf{Additional Experiments: Limitation of FID.} In addition, we have added a figure showcasing a t-SNE visualization for a minority class as an example, as shown in~\figref{fig:tsne}, to further illustrate the limitation of FID we mentioned above. As shown in the figure: 
\begin{itemize}[nosep,leftmargin=15pt]
    \item There are two ground-truth IDs (i.e., two ground-truth individuals) in the training set. 
    \item Our PoGDiff can successfully generate images similar to these two ground-truth ID while maintaining diversity.
    \item All baselines, including CBDM, fail to generate accurate images according to the ground-truth IDs. In fact most generated images from the baselines are similar to other IDs, i.e., generating the facial images of wrong individuals.
\end{itemize}
These results show that:
\begin{itemize}[nosep,leftmargin=15pt]
    \item Our PoGDiff significantly outperforms the baselines.
    \item FID fails to capture such improvements because it depends only on the mean and variance of the distribution, losing a lot of information during evaluation.
\end{itemize}}

{\textbf{FID Measures Both ID Consistency and Diversity.} Note that our FID is computed \emph{for each ID separately}, and the final FID score in the tables (e.g.,~\tabref{table:total-fid}) is the average FID over all IDs. Therefore FID measures both ID consistency and diversity.} 

To see why, note that the FID score measures the distance between two Gaussian distributions, where the \emph{mean} of the Gaussian represents the \emph{identity (ID)} and the \emph{variance} represents the \emph{diversity}. For example, the \emph{mean} of the ground-truth distribution represents the embedding position of the ground-truth ID, while the \emph{variance} of the ground-truth distribution represents the \emph{diversity} of ground-truth images, and similarly for the generated images. 

Therefore, a lower FID score indicates that the generated-image distribution better matches the ground-truth distribution \textbf{in terms of both ID and diversity}.

{\subsection{Which Distribution the Model Converges to After Training}
Our PoG objective introduces an additional term that encourages consistency across semantically similar conditioning prompts, but does not fundamentally alter the underlying diffusion process. While full convergence guarantees would be an interesting work, such analyses are rare in conditional diffusion literature. Since our formulation preserves the standard denoising score matching structure, its convergence behavior broadly follows that of existing diffusion models.}

{\subsection{The Variance of the Predicted Distribution} \label{app:psi_var}
\begin{itemize}
    \item In our method, our predicted variance is indeed conditioned on $\psi$, and $\psi$ is conditioned on $\y$, as mentioned in~\eqnref{eq:pog_simple}.
    \item The definition of $\psi$ relies on the assumption that data belongs to the same class or person (i.e., $\mathcal{I}(\cdot)$). Although such information might not implicitly available for some datasets, We can use a pretrained image classifier (e.g., ResNet and ViT) to obtain the label. Alternatively, we can also use clustering method and treat the cluster ID as the class label.
    \item Although the definition of $\psi$ relies on the VAE training. In our experiment, we used a simple VAE with three-layer MLPs for both the encoder and decoder. We found that variations in architecture, learning rate, and number of training epochs had little effect on the final fine-tuning performance since it is very efficient and easy to train, it only costs around a few minutes for a single GPU.
\end{itemize}}

\subsection{Computational Efficiency}
For the computation efficiency of our method:
\begin{itemize}
    \item\textbf{Training:} Since our denoising step requires two forward passes of the denoising model, the runtime is approximately twice that of Stable Diffusion (SD)~\citep{rombach2022high}. However, we observe that under the same training time budget (e.g., SD for $12k$ steps vs. PoGDiff for $6k$ steps), our method is already able to generate accurate images for tail classes, while Stable Diffusion continues to produce samples biased toward head classes.
    \item\textbf{Inference:} During inference, our generation process remains identical to that of Stable Diffusion and thus incurs no computational overhead.
\end{itemize}

{\subsection{Discussions on Failure Cases}
The reason why some generated individuals do not match expectations lies in the selection of $y'$. Since our text prompts are fully end-to-end generated from LLAVA-NEXT, different individuals may have higher similarity scores under $\psi$ than the true target. For example, Einstein at age 41 ($y$="The image is a black and white photograph of a man named Einstein, who has a mustache, curly hair, and a pipe in his mouth.") might appear more similar - based on textual semantics - to CHARLIERUGGLES at age 48 ($y'_1$="CHARLIERUGGLES is a man is a man with a mustache, holding a pipe in his mouth.") than to Einstein at age 24 ($y'_2$="The image is a black and white portrait of a man named Einstein, who has a mustache and is wearing a suit."). As a result, PoGDiff may use $y'_1$ rather than $y'_2$ as the neighbor $y'$ for $y$. }

{In general, when the dataset includes many individuals with overlapping semantic traits, it is possible that in some steps $y'$ is not the same individual as $y$, leading to subtle deviations in facial details. These deviations are typically small but beneficial; they help improve diversity while still preserving overall identity correctness.}

\section{Impact Statement}
{Finetuning under imbalanced datasets in specific domain presents an inescapable challenge in generative AI. For example, when generating the counterfactual outcomes for users with specific actions, such ``user(or patient)--action--outcome'' pairs are always imbalanced, as it is impossible for any company or any hospital to obtains all the pairs. As such, to save the budget, learning the mapping from ``user(or patient)--action'' (sentence description) to ``outcome'' (images) is where this challenge is particularly pronounced. 
Our proposed method, PoGDiff, represents an innovative and efficient solution to navigate this issue. We argue that the complexity and importance of this problem warrant further research, given its profound implications across diverse fields. This exploration not only advances our understanding but also opens new avenues for significant impact, underscoring the need for continued investigation into training generative models under imbalanced datasets.}

\newpage
\section*{NeurIPS Paper Checklist}

The checklist is designed to encourage best practices for responsible machine learning research, addressing issues of reproducibility, transparency, research ethics, and societal impact. Do not remove the checklist: {\bf The papers not including the checklist will be desk rejected.} The checklist should follow the references and follow the (optional) supplemental material.  The checklist does NOT count towards the page
limit. 

Please read the checklist guidelines carefully for information on how to answer these questions. For each question in the checklist:
\begin{itemize}
    \item You should answer \answerYes{}, \answerNo{}, or \answerNA{}.
    \item \answerNA{} means either that the question is Not Applicable for that particular paper or the relevant information is Not Available.
    \item Please provide a short (1–2 sentence) justification right after your answer (even for NA). 
\end{itemize}

{\bf The checklist answers are an integral part of your paper submission.} They are visible to the reviewers, area chairs, senior area chairs, and ethics reviewers. You will be asked to also include it (after eventual revisions) with the final version of your paper, and its final version will be published with the paper.

The reviewers of your paper will be asked to use the checklist as one of the factors in their evaluation. While "\answerYes{}" is generally preferable to "\answerNo{}", it is perfectly acceptable to answer "\answerNo{}" provided a proper justification is given (e.g., "error bars are not reported because it would be too computationally expensive" or "we were unable to find the license for the dataset we used"). In general, answering "\answerNo{}" or "\answerNA{}" is not grounds for rejection. While the questions are phrased in a binary way, we acknowledge that the true answer is often more nuanced, so please just use your best judgment and write a justification to elaborate. All supporting evidence can appear either in the main paper or the supplemental material, provided in appendix. If you answer \answerYes{} to a question, in the justification please point to the section(s) where related material for the question can be found.

IMPORTANT, please:
\begin{itemize}
    \item {\bf Delete this instruction block, but keep the section heading ``NeurIPS Paper Checklist"},
    \item  {\bf Keep the checklist subsection headings, questions/answers and guidelines below.}
    \item {\bf Do not modify the questions and only use the provided macros for your answers}.
\end{itemize}


\begin{enumerate}

\item {\bf Claims}
    \item[] Question: Do the main claims made in the abstract and introduction accurately reflect the paper's contributions and scope?
    \item[] Answer: \answerYes{} 
    \item[] Justification: It can be found in the~\secref{sec:intro}.

\item {\bf Limitations}
    \item[] Question: Does the paper discuss the limitations of the work performed by the authors?
    \item[] Answer: \answerYes{} 
    \item[] Justification: In~\appref{app:limitation}.

\item {\bf Theory assumptions and proofs}
    \item[] Question: For each theoretical result, does the paper provide the full set of assumptions and a complete (and correct) proof?
    \item[] Answer: \answerYes{} 
    \item[] Justification: Although we do not have theoretical result like a theory paper, but we provide the proof for our lemma in appendix.

    \item {\bf Experimental result reproducibility}
    \item[] Question: Does the paper fully disclose all the information needed to reproduce the main experimental results of the paper to the extent that it affects the main claims and/or conclusions of the paper (regardless of whether the code and data are provided or not)?
    \item[] Answer: \answerYes{} 
    \item[] Justification: Yes, they are all discussed in main paper and appendix.

\item {\bf Open access to data and code}
    \item[] Question: Does the paper provide open access to the data and code, with sufficient instructions to faithfully reproduce the main experimental results, as described in supplemental material?
    \item[] Answer: \answerNo{} 
    \item[] Justification: They are all public, and detailed information are in the appendix. For the code, we will release it once this paper is accepted.

\item {\bf Experimental setting/details}
    \item[] Question: Does the paper specify all the training and test details (e.g., data splits, hyperparameters, how they were chosen, type of optimizer, etc.) necessary to understand the results?
    \item[] Answer: \answerYes{} 
    \item[] Justification: Yes, they are all discussed in the appendix.

\item {\bf Experiment statistical significance}
    \item[] Question: Does the paper report error bars suitably and correctly defined or other appropriate information about the statistical significance of the experiments?
    \item[] Answer: \answerNA{} 
    \item[] Justification: The error bars are not applicable in our settings.

\item {\bf Experiments compute resources}
    \item[] Question: For each experiment, does the paper provide sufficient information on the computer resources (type of compute workers, memory, time of execution) needed to reproduce the experiments?
    \item[] Answer: \answerYes{} 
    \item[] Justification: They are all discussed in the appendix.
    
\item {\bf Code of ethics}
    \item[] Question: Does the research conducted in the paper conform, in every respect, with the NeurIPS Code of Ethics \url{https://neurips.cc/public/EthicsGuidelines}?
    \item[] Answer: \answerNA{} 
    \item[] Justification: It is not applicable to our settings.

\item {\bf Broader impacts}
    \item[] Question: Does the paper discuss both potential positive societal impacts and negative societal impacts of the work performed?
    \item[] Answer: \answerYes{} 
    \item[] Justification: They are discussed in the appendix.
    
\item {\bf Safeguards}
    \item[] Question: Does the paper describe safeguards that have been put in place for responsible release of data or models that have a high risk for misuse (e.g., pretrained language models, image generators, or scraped datasets)?
    \item[] Answer: \answerNA{} 
    \item[] Justification: It is not applicable to our settings.

\item {\bf Licenses for existing assets}
    \item[] Question: Are the creators or original owners of assets (e.g., code, data, models), used in the paper, properly credited and are the license and terms of use explicitly mentioned and properly respected?
    \item[] Answer: \answerYes{} 
    \item[] Justification: We cite all the references.

\item {\bf New assets}
    \item[] Question: Are new assets introduced in the paper well documented and is the documentation provided alongside the assets?
    \item[] Answer: \answerNA{} 
    \item[] Justification: It is not applicable to our settings.

\item {\bf Crowdsourcing and research with human subjects}
    \item[] Question: For crowdsourcing experiments and research with human subjects, does the paper include the full text of instructions given to participants and screenshots, if applicable, as well as details about compensation (if any)? 
    \item[] Answer: \answerNA{} 
    \item[] Justification: It is not applicable to our settings.

\item {\bf Institutional review board (IRB) approvals or equivalent for research with human subjects}
    \item[] Question: Does the paper describe potential risks incurred by study participants, whether such risks were disclosed to the subjects, and whether Institutional Review Board (IRB) approvals (or an equivalent approval/review based on the requirements of your country or institution) were obtained?
    \item[] Answer: \answerNA{} 
    \item[] Justification: It is not applicable to our settings.

\item {\bf Declaration of LLM usage}
    \item[] Question: Does the paper describe the usage of LLMs if it is an important, original, or non-standard component of the core methods in this research? Note that if the LLM is used only for writing, editing, or formatting purposes and does not impact the core methodology, scientific rigorousness, or originality of the research, declaration is not required.
    \item[] Answer: \answerNA{} 
    \item[] Justification: It is not applicable to our settings.

\end{enumerate}

\end{document}